\documentclass{article}
\usepackage[preprint]{style/icml2026} 

\usepackage{microtype}
\usepackage{graphicx}
\usepackage{subcaption}
\usepackage{booktabs}
\usepackage{listings}
\usepackage{fix-cm}

\usepackage{amsmath,amssymb,mathtools,amsthm}

\usepackage{algorithm}
\usepackage{algorithmic}
\usepackage{float}

\usepackage{hyperref}
\usepackage[capitalize,noabbrev]{cleveref}

\usepackage{paralist}
\usepackage{enumitem}

\usepackage{tcolorbox}

\newif\ifanonymized
\anonymizedfalse 

\usepackage[disable,textsize=tiny]{todonotes}

\theoremstyle{plain}
\newtheorem{theorem}{Theorem}[section]
\newtheorem{proposition}[theorem]{Proposition}

\theoremstyle{definition}
\newtheorem{definition}[theorem]{Definition}

\theoremstyle{remark}

\usepackage{multirow}
\icmltitlerunning{Not All Invariants Are Equal}

\begin{document}
\twocolumn[

\icmltitle{Not All Invariants Are Equal: Curating Training Data to Accelerate Program Verification with SLMs}

\begin{icmlauthorlist}
  \icmlauthor{Ido Pinto}{huji}
  \icmlauthor{Yizhak Yisrael Elboher}{huji}
  \icmlauthor{Haoze Wu}{amherst,vmware}
  \icmlauthor{Nina Narodytska}{vmware}
  \icmlauthor{Guy Katz}{huji}
\end{icmlauthorlist}

\icmlaffiliation{huji}{Hebrew University of Jerusalem, Israel}
\icmlaffiliation{amherst}{Amherst College, USA}
\icmlaffiliation{vmware}{VMware Research by Broadcom, USA}
\icmlcorrespondingauthor{Ido Pinto}{ido.pinto@mail.huji.ac.il}
\icmlcorrespondingauthor{Guy Katz}{g.katz@mail.huji.ac.il}

\icmlkeywords{Machine Learning, Data Curation, Data Augmentation, Formal Verification, Large Language Models}

\vskip 0.3in
]

\printAffiliationsAndNotice{}

\begin{abstract}
The synthesis of inductive loop invariants remains a critical bottleneck in automated program verification. While Large Language Models (LLMs) show promise in mitigating this issue, they often fail on complex programs, producing invariants that are invalid or computationally ineffective. Although fine-tuning is a natural strategy to address these limitations, obtaining high-quality training data remains an open challenge. We first formalize the properties required for a high-quality training invariant, and then present \textsc{Wonda}, a rigorous data curation pipeline that extracts such invariants from raw verifier output via AST-based normalization followed by LLM-driven semantic rewriting and augmentation with provable quality guarantees. Fine-tuning Small Language Models (SLMs) on \textsc{Wonda}-curated data yields consistent gains across the Qwen3, Llama-3.1, and Mistral families: the 4B and 8B Qwen3 models nearly double invariant correctness and double speedup rates, while Llama-3.1-8B triples both. On the challenging InvBench suite, the same 4B model outperforms an off-the-shelf model 20$\times$ its size and matches the end-to-end verification time of GPT-OSS-120B, while a 14B Qwen3 model matches that of the frontier model GPT-5.2, all without test-time compute overhead. Our code is publicly available on \href{https://github.com/idopinto/wonda}{GitHub}.
\end{abstract}
\section{Introduction}
Automated program verification is a powerful technique for ensuring the reliability of critical software infrastructure. At the core of deductive verification lies the challenge of loop invariant synthesis: identifying a logical property that is preserved across every iteration of a loop and is strong enough to prove the program's correctness. Despite decades of research into symbolic methods such as Craig interpolation~\citep{mcmillan2003interpolation} and constraint solving~\citep{colon2003linear, gupta2009invgen}, finding inductive invariants remains a challenging problem, and a practical bottleneck for modern verifiers.

The advent of Large Language Models (LLMs) has introduced a new paradigm: \emph{generate-and-verify}. Models such as GPT-4~\citep{openai2023gpt4} and Claude~\citep{anthropic2024claude} have demonstrated an ability to hypothesize loop invariants~\citep{pei2023can}. Recent tools such as Loopy~\citep{kamath2024leveraging}, Lemur~\citep{wu2024lemur} and ACInv~\citep{liu2025acinv} leverage this capability, wrapping LLMs in iterative refinement strategies that filter hallucinations using SMT solvers.

Despite this progress, a significant gap remains. As observed by~\citet{wei2025invbench}, while LLM-based verifiers represent a promising direction, they do not yet offer a significant advantage over state-of-the-art symbolic tools, such as UAutomizer, which do not leverage LLMs.

Given these limitations, the natural question arises: can we \emph{train} an LLM that specializes in the task of invariant generation? Such specialization is standard practice in many domains and has proven successful. Indeed, within the program verification literature, this specific problem has received increasing attention~\citep{pei2023can,wei2025invbench}. However, the success of these training efforts has only been partial. Recent work on fine-tuning invariant generation models reports improvements on easy verification problems, yet achieving significant speedups on hard benchmarks remains elusive~\citep{wei2025invbench}.
We argue that data quality, rather than model scale alone, is the key bottleneck for current neural invariant synthesizers. Existing datasets~\citep{wei2025invbench} rely on raw outputs from symbolic tools (e.g.,\ UAutomizer~\citep{heizmann2013ultimate}), which suffer from two limitations:
\begin{enumerate}
    \item \emph{Low Pedagogical Value:} Solver-generated invariants are often technically correct but structurally obfuscated (see Figure~\ref{fig:raw_example}), and each output reflects only one valid choice among many for the same program. As a result, models may learn verifier-specific artifacts, rather than the underlying program logic or the broader space of useful invariants.
    \item \emph{Dependence on Solver Outputs:} When training relies heavily on solver-generated data, the model may inherit the solver's biases, making it harder to improve beyond the current symbolic state-of-the-art.

\end{enumerate}
To mitigate these limitations, we introduce \textsc{Wonda}, a rigorous data curation pipeline. Instead of directly training on verifier-generated invariants, \textsc{Wonda} explicitly optimizes for learnability. We employ Abstract Syntax Tree (AST) normalization to clean the structural representation of invariants. We further implement an LLM-driven semantic rewriting and augmentation engine that transforms obscure machine-generated invariants into concise, interpretable forms, simplifying their logical semantics while maintaining soundness. Augmentation generates multiple candidate rewrites per input, so the model is not limited to the single invariant the solver originally produced. To ensure soundness, we use a formal verification tool to validate invariants after any transformation step that may not preserve correctness.

\paragraph{Our Contributions.}
\begin{enumerate}
    \item We demonstrate that directly fine-tuning invariants generated by symbolic solvers does not reliably improve model performance and can even degrade it.
    \item We propose \textsc{Wonda}, a rigorous data curation framework that combines AST normalization with LLM-driven semantic rewriting and augmentation, transforming raw, obfuscated symbolic outputs into high-quality training signals while preserving soundness.
    \item We present a thorough experimental evaluation across the Qwen3, Llama-3.1, and Mistral model families, showing that \textsc{Wonda}-curated data lets small fine-tuned models match much larger and frontier models on end-to-end verification time, without test-time compute overhead.
\end{enumerate}


\definecolor{kw}{RGB}{160, 32, 160}     
\definecolor{op}{RGB}{180, 60, 60}      
\definecolor{num}{RGB}{20, 100, 180}    
\definecolor{id}{RGB}{40, 90, 40}       

\usetikzlibrary{shapes, arrows, positioning, fit, calc}

\begin{figure}[htb!]
\centering

\begin{tikzpicture}[
    codebox/.style={
        draw, 
        rounded corners=4pt, 
        font=\ttfamily\fontsize{6.8}{8}\selectfont, 
        align=justify, 
        fill=gray!2,   
        inner sep=6pt
    },
    sublbl/.style={font=\fontsize{7.5}{9}\selectfont\bfseries, text=gray!90}
]

\node[codebox, text width=0.98\columnwidth] (v0) {%
((((((((((((((((((( \textcolor{num}{1} \textcolor{op}{<=} \textcolor{id}{weight1} ) \textcolor{op}{\&\&} (((( \textcolor{kw}{long\ long} ) \textcolor{id}{weight1} \textcolor{op}{+} \textcolor{id}{weight2} ) \textcolor{op}{+} \textcolor{num}{1} ) \textcolor{op}{<=} \textcolor{id}{max\_threshold} )) \textcolor{op}{\&\&} ( \textcolor{id}{weight2} \textcolor{op}{<=} \textcolor{num}{10} )) \textcolor{op}{\&\&} ( \textcolor{num}{20} \textcolor{op}{<=} (( \textcolor{kw}{\_\_int128} ) (( \textcolor{kw}{long\ long} ) \textcolor{id}{weight1} \textcolor{op}{*} \textcolor{num}{14} ) \textcolor{op}{+} \textcolor{id}{cumulative\_weight} )))) \textcolor{op}{\&\&} ( \textcolor{id}{weight1} \textcolor{op}{<=} \textcolor{num}{10} ) \textcolor{op}{\&\&} ( \textcolor{id}{count} \textcolor{op}{<=} \textcolor{num}{4} )) \textcolor{op}{\&\&} ( \textcolor{num}{1} \textcolor{op}{<=} \textcolor{id}{weight2} )) \textcolor{op}{\&\&} (( \textcolor{kw}{long\ long} ) \textcolor{id}{weight1} \textcolor{op}{+} \textcolor{id}{weight2} \textcolor{op}{<=} \textcolor{id}{cumulative\_weight} )) \textcolor{op}{||} \dots\ \textcolor{op}{||} (((((((( \textcolor{num}{1} \textcolor{op}{<=} \textcolor{id}{weight1} ) \textcolor{op}{\&\&} (((( \textcolor{kw}{long\ long} ) \textcolor{id}{weight1} \textcolor{op}{+} \textcolor{id}{weight2} ) \textcolor{op}{+} \textcolor{num}{1} ) \textcolor{op}{<=} \textcolor{id}{max\_threshold} )) \textcolor{op}{\&\&} ( \textcolor{id}{count} \textcolor{op}{<=} \textcolor{num}{5} ) \textcolor{op}{\&\&} ( \textcolor{id}{weight2} \textcolor{op}{<=} \textcolor{num}{10} )) \textcolor{op}{\&\&} ( \textcolor{id}{weight1} \textcolor{op}{<=} \textcolor{num}{10} )) \textcolor{op}{\&\&} ( \textcolor{num}{20} \textcolor{op}{<=} (( \textcolor{kw}{\_\_int128} ) \textcolor{id}{cumulative\_weight} \textcolor{op}{+} (( \textcolor{kw}{long\ long} ) \textcolor{id}{weight1} \textcolor{op}{*} \textcolor{num}{13} )))) \textcolor{op}{\&\&} ( \textcolor{num}{1} \textcolor{op}{<=} \textcolor{id}{weight2} )) \textcolor{op}{\&\&} (( \textcolor{kw}{long\ long} ) \textcolor{id}{weight1} \textcolor{op}{+} \textcolor{id}{weight2} \textcolor{op}{<=} \textcolor{id}{cumulative\_weight} )))
};

\node[sublbl, below=0.08cm of v0] {3084 chars, 75 conjuncts, 11 disjuncts};
\end{tikzpicture}

\caption{Example of a raw verbose invariant generated by UAutomizer. \textsc{Wonda} transforms such outputs into more compact, learnable forms.}
\label{fig:raw_example}
\end{figure}

\definecolor{kw}{RGB}{160, 32, 160}     
\definecolor{op}{RGB}{180, 60, 60}      
\definecolor{num}{RGB}{20, 100, 180}    
\definecolor{id}{RGB}{40, 90, 40}       

\begin{figure*}[ht!]
\centering

\begin{tikzpicture}[
    codebox/.style={
        draw, 
        rounded corners=4pt, 
        font=\ttfamily\fontsize{7}{8.5}\selectfont, 
        align=left, 
        fill=gray!2, 
        inner sep=5pt
    },
    transbox/.style={
        draw, 
        rounded corners=4pt, 
        font=\ttfamily\fontsize{7}{8.5}\selectfont, 
        align=left, 
        fill=#1, 
        inner sep=5pt
    },
    llmbox/.style={
        draw, 
        rounded corners=5pt, 
        fill=violet!15,        
        font=\footnotesize\bfseries, 
        minimum height=0.6cm, 
        minimum width=1.2cm
    },
    vbsbox/.style={
        draw, 
        rounded corners=6pt, 
        fill=red!15,            
        font=\small\bfseries, 
        minimum width=1.2cm, 
        minimum height=0.75cm
    },
    lbl/.style={font=\fontsize{7.5}{9}\selectfont\bfseries},
    timelbl/.style={font=\tiny, fill=white, inner sep=1pt}
]

\node[codebox] (input) {
\textcolor{kw}{int} \textcolor{id}{x} \textcolor{op}{=} \textcolor{num}{0};\\
\textcolor{kw}{int} \textcolor{id}{y} \textcolor{op}{=} \textcolor{num}{100};\\
\textcolor{kw}{assume}(\textcolor{id}{x} \textcolor{op}{==} \textcolor{num}{0} \textcolor{op}{\&\&} \textcolor{id}{y} \textcolor{op}{==} \textcolor{num}{100}); \textcolor{darkgray}{$\triangleright$ $p$}\\
\textcolor{kw}{while} (\textcolor{id}{y} \textcolor{op}{>} \textcolor{num}{0}) \{\\
\quad \textcolor{gray}{// Location $l$}\\
\quad \textcolor{id}{x} \textcolor{op}{+=} \textcolor{num}{3};\\
\quad \textcolor{id}{y} \textcolor{op}{-=} \textcolor{num}{5};\\
\}\\
\textcolor{kw}{assert}(\textcolor{id}{x} \textcolor{op}{>} \textcolor{id}{y}); \textcolor{darkgray}{$\triangleright$ $q$}\\
};
\node[lbl, above=0.1cm of input] {Input Query: $\langle A, P, q \rangle$ where $A=\{p\}$};

\node[transbox=blue!9, right=2.2cm of input] (baseline) {
\textcolor{kw}{int} \textcolor{id}{x} \textcolor{op}{=} \textcolor{num}{0};\\
\textcolor{kw}{int} \textcolor{id}{y} \textcolor{op}{=} \textcolor{num}{100};\\
\textcolor{kw}{assume}(\textcolor{id}{x} \textcolor{op}{==} \textcolor{num}{0} \textcolor{op}{\&\&} \textcolor{id}{y} \textcolor{op}{==} \textcolor{num}{100}); \textcolor{darkgray}{$\triangleright$ $p$}\\
\textcolor{kw}{while} (\textcolor{id}{y} \textcolor{op}{>} \textcolor{num}{0}) \{\\
\quad \textcolor{id}{x} \textcolor{op}{+=} \textcolor{num}{3};\\
\quad \textcolor{id}{y} \textcolor{op}{-=} \textcolor{num}{5};\\
\}\\
\textcolor{kw}{assert}(\textcolor{id}{x} \textcolor{op}{>} \textcolor{id}{y}); \textcolor{darkgray}{$\triangleright$ $q$}\\
};
\node[lbl, above=0.1cm of baseline] {Baseline: $\mathcal{V}(A,\, P,\, q)$};

\node[llmbox, below=0.6cm of input] (llm) {LLM};

\node[font=\scriptsize, fill=orange!14, draw, rounded corners=2pt, inner sep=2.5pt, below=0.1cm of llm] (prop) {$I = (l,\, 5x + 3y == 300)$};

\node[transbox=teal!9, right=1.0cm of llm, yshift=-2.2cm] (correct) {
\textcolor{kw}{int} \textcolor{id}{x} \textcolor{op}{=} \textcolor{num}{0};\\
\textcolor{kw}{int} \textcolor{id}{y} \textcolor{op}{=} \textcolor{num}{100};\\
\textcolor{kw}{assume}(\textcolor{id}{x} \textcolor{op}{==} \textcolor{num}{0} \textcolor{op}{\&\&} \textcolor{id}{y} \textcolor{op}{==} \textcolor{num}{100}); \textcolor{darkgray}{$\triangleright$ $p$}\\
\textcolor{kw}{while} (\textcolor{id}{y} \textcolor{op}{>} \textcolor{num}{0}) \{\\
\quad \textcolor{kw}{assert}(\textcolor{num}{5}\textcolor{op}{*}\textcolor{id}{x} \textcolor{op}{+} \textcolor{num}{3}\textcolor{op}{*}\textcolor{id}{y} \textcolor{op}{==} \textcolor{num}{300}); \textcolor{darkgray}{$\triangleright$ $I$}\\
\quad \textcolor{id}{x} \textcolor{op}{+=} \textcolor{num}{3};\\
\quad \textcolor{id}{y} \textcolor{op}{-=} \textcolor{num}{5};\\
\}\\
};

\node[transbox=teal!9, right=0.5cm of correct] (useful) {
\textcolor{kw}{int} \textcolor{id}{x} \textcolor{op}{=} \textcolor{num}{0};\\
\textcolor{kw}{int} \textcolor{id}{y} \textcolor{op}{=} \textcolor{num}{100};\\
\textcolor{kw}{assume}(\textcolor{id}{x} \textcolor{op}{==} \textcolor{num}{0} \textcolor{op}{\&\&} \textcolor{id}{y} \textcolor{op}{==} \textcolor{num}{100}); \textcolor{darkgray}{$\triangleright$ $p$}\\
\textcolor{kw}{while} (\textcolor{id}{y} \textcolor{op}{>} \textcolor{num}{0}) \{\\
\quad \textcolor{kw}{assume}(\textcolor{num}{5}\textcolor{op}{*}\textcolor{id}{x} \textcolor{op}{+} \textcolor{num}{3}\textcolor{op}{*}\textcolor{id}{y} \textcolor{op}{==} \textcolor{num}{300}); \textcolor{darkgray}{$\triangleright$ $I$}\\
\quad \textcolor{id}{x} \textcolor{op}{+=} \textcolor{num}{3};\\
\quad \textcolor{id}{y} \textcolor{op}{-=} \textcolor{num}{5};\\
\}\\
\textcolor{kw}{assert}(\textcolor{id}{x} \textcolor{op}{>} \textcolor{id}{y}); \textcolor{darkgray}{$\triangleright$ $q$}\\
};

\node[lbl, yshift=+0.15cm] at (correct.north |- llm.east) (lbl_correct) {Correctness Check: $\mathcal{V}_1 = \mathcal{V}(A,\, P,\, I)$};
\node[lbl, yshift=+0.15cm] at (useful.north |- llm.east) (lbl_useful) {Sufficiency Check: $\mathcal{V}_2 = \mathcal{V}(A \cup \{I\},\, P,\, q)$};

\node[vbsbox, right=1.0cm of useful] (vbs) {VBS};

\draw[->, >=stealth, thick, gray] (input.east) -- (baseline.west);
\draw[->, >=stealth, thick, gray] (input.south) -- (llm.north)
    node[timelbl, right=0.04cm of llm, yshift=-0.15cm] {$t_m$};

\draw[->, >=stealth, thick, violet!50] (llm.east) -- (useful.north |- llm.east) -- (useful.north);
\draw[->, >=stealth, thick, violet!50] (correct.north |- llm.east) -- (correct.north);

\draw[->, >=stealth, thick, gray] (baseline.east) -| (vbs.north) node[timelbl, pos=0.3, above] {$t_b$};
\draw[->, >=stealth, thick, violet!50] (useful.east) -- (vbs.west) node[timelbl, pos=0.5, above] {$t_2$};
\draw[->, >=stealth, thick, violet!50] (correct.south) -- ++(0,-0.5) -| (vbs.south) node[timelbl, pos=0.2, below] {$t_1$};

\end{tikzpicture}

\caption{Illustration of the VBS metric. Given a query $\langle A, P, q \rangle$, the LLM proposes invariant $I=(l, \varphi)$ with inference latency $t_m$. $\mathcal{V}_1$ and $\mathcal{V}_2$ run in parallel alongside the baseline $\mathcal{V}(A, P, q)$, with $t_v = \max(t_1, t_2)$. VBS selects $\min(t_v, t_b)$ if $I$ is correct and the decomposed checks yield a conclusive result (per Table 1), and falls back to $t_b$ otherwise. For the end-to-end metric $\text{VBP}_{\text{E2E}}$, $t_m$ is added to $t_v$ for sufficient instances.}
\label{fig:vbs}
\end{figure*}
\section{Related Work}
\label{sec:related}

\textbf{Traditional Invariant Synthesis.}
Invariant synthesis has been extensively studied through various formalisms. Abstract Interpretation~\citep{cousot1977abstract} gave rise to techniques for discovering specific classes of invariants, such as affine relationships~\citep{karr1976affine} and linear restraints~\citep{cousot1978automatic} among variables. Model Checking~\citep{clarke2018introduction} and Predicate Abstraction~\citep{flanagan2002predicate, lahiri2007predicate} infer invariants by refining abstract states based on a set of predicates, while Craig Interpolation~\citep{mcmillan2003interpolation} generates loop invariants from proofs of unsatisfiability in bounded model checking traces. Tools like UAutomizer~\citep{heizmann2013ultimate} and Eldarica~\citep{hojjat2018eldarica} build on these approaches. A separate line of work casts invariant generation as a constraint-satisfaction problem solved with off-the-shelf solvers~\citep{colon2003linear,gupta2009invgen, fedyukovich2018accelerating}. Finally, dynamic analysis tools like Daikon~\citep{ernst2007daikon} infer likely invariants by observing program execution traces; while efficient, dynamic methods are unsound, as they can only guarantee correctness for the observed executions.

\textbf{Learning-based Invariant Synthesis.}
Early techniques for data-driven invariant synthesis relied on decision trees~\citep{garg2016learning}, algebraic inference~\citep{sharma2013data}, and Horn-ICE learning~\citep{ezudheen2018horn}. \citet{pei2023can} explored fine-tuning LLMs for invariant prediction, training on Daikon~\citep{ernst2007daikon} outputs. With the rise of code-capable LLMs, more recent frameworks employ iterative \emph{generate-and-verify} loops, using symbolic solvers to filter or refine LLM proposals: Loopy~\citep{kamath2024leveraging} applies Houdini-based filtering of LLM-generated candidates, LaM4Inv~\citep{wu2024llm} iteratively queries the LLM and uses BMC to filter and reassemble candidate predicates across rounds, LEMUR~\citep{wu2024lemur} introduces backtracking to repair invalid invariants. Related directions include contrastive ranking~\citep{chakraborty2023ranking}, C++ class invariants~\citep{sun2025classinvgen}, and complex loop structures~\citep{liu2025acinv}.

\textbf{Fine-tuning for Invariant Synthesis.} Fine-tuning LLMs to propose a single invariant, subsequently verified by a symbolic solver, has received increasing attention. Unlike \citet{pei2023can}, who train on Daikon outputs without formal verification of the generated invariants, \citet{wei2025invbench} introduce a one-shot setting using UAutomizer-generated training data, which is formally correct but, as we show, structurally noisy (see Figure~\ref{fig:raw_example}). We adopt the same evaluation framework and focus on improving training data quality and diversity, producing more compact and generalizable invariants. The one-shot and iterative directions are complementary: while iterative tools such as Lemur~\citep{wu2024lemur} and Loopy~\citep{kamath2024leveraging} refine invariants over multiple LLM calls, a stronger one-shot generator provides better initialization for such loops.
\section{Preliminaries}

We ground our approach in Hoare logic \citep{hoare1969axiomatic}, a seminal framework for reasoning about program correctness. A Hoare triple $\{p\}\, S \,\{q\}$ asserts that if precondition $p$ holds before executing statement $S$, then postcondition $q$ holds afterward. While sequential statements are straightforward to compose, loops pose a fundamental challenge: they may execute for an unbounded number of iterations, requiring an auxiliary predicate, a \emph{loop invariant}, to enable finite reasoning.

For a loop \texttt{while}~$B$~\texttt{do}~$S$ with precondition $p$ and postcondition $q$, a valid invariant $I$ must satisfy:
\begin{inparaenum}[(i)]
    \item \emph{Initiation:} $p \Rightarrow I$, meaning the invariant holds upon loop entry;
    \item \emph{Consecution:} $\{I \land B\}\, S \,\{I\}$, meaning the invariant is preserved by each iteration; and
    \item \emph{Sufficiency:} $I \land \neg B \Rightarrow q$, meaning that upon termination, the invariant implies the postcondition.
\end{inparaenum}
A predicate satisfying both initiation and consecution is termed an \emph{inductive invariant}. Throughout this paper, we refer to this property as \emph{correctness} and use the terms interchangeably; if the predicate additionally satisfies sufficiency, it constitutes a formal proof of the postcondition $q$.

\paragraph{Running Example.}
We illustrate these concepts using the program shown in Figure~\ref{fig:vbs} (top-left block). The program initializes $x = 0$ and $y = 100$ (the precondition $p$), then repeatedly increments $x$ by 3 and decrements $y$ by 5 until $y \le 0$. The verification goal is to prove the postcondition $q \equiv x > y$.
We claim that the candidate invariant $I\equiv5x+3y=300$ satisfies these criteria. It is \emph{inductive} because it holds initially $(5(0)+3(100)=300)$ and is preserved by the loop body updates: assuming $I$ holds, the new state satisfies $5(x+3)+3(y-5)=5x+15+3y-15=5x+3y=300$.
Furthermore, the invariant is \emph{sufficient}: upon loop termination, $\neg B\equiv y\le0$ combined with $I$ implies $5x + 3y = 300 \land y \le 0$. Substituting, $5x = 300 - 3y \ge 300$ (since $-3y \ge 0$), so $x \ge 60$. As $x \ge 60 > 0 \ge y$, the postcondition $q \equiv x > y$ follows.

\subsection{Program Verification}
Let $P$ denote a program, let $\mathrm{Loc}(P)$ denote its set of program locations, and let $\mathcal{L}(P) \subseteq \mathrm{Loc}(P)$ denote its loop entry locations.
\begin{definition}[Property]
    A \emph{property} is a pair $(l, \varphi)$, where $l\in \mathrm{Loc}(P)$ and $\varphi$ is a predicate over the variables of program $P$.  
    For a given execution of $P$,
    a property \emph{holds}, or is \emph{satisfied}, if $\varphi$ \emph{holds}, whenever $P$ reaches line $l$. 
\end{definition}

\begin{definition}[Verification Query]\label{def:ver_oracle}
A \emph{verification query} is a triple $\langle A, P, q\rangle$ where $P$ is a program, $q$ is a property (the postcondition), and $A$ is a set of properties representing preconditions.
\end{definition}

\begin{definition}[Verification Query Validity]\label{def:validity}
A verification query $\langle A, P, q\rangle$ is \emph{valid} if every execution of $P$ that satisfies the (precondition) properties in $A$ also satisfies the postcondition $q$.
\end{definition}

\begin{definition}[Verification Oracle]
A \emph{verification oracle} $\mathcal{V}$ takes a verification query $\langle A, P, r \rangle$ and returns:
\[
    \mathcal{V}(A, P, r) \in \{\textsc{True}, \textsc{False}, \textsc{Unknown}\}
\]
where \textsc{True} indicates the query is valid, \textsc{False} indicates a counterexample exists, and \textsc{Unknown} indicates the oracle could not determine the result (e.g., due to timeout or inherent incompleteness).
Here, we consider only \emph{sound} oracles: if $\mathcal{V}(A, P, r) = \textsc{True}$, then the query is valid; if $\mathcal{V}(A, P, r) = \textsc{False}$, then a counterexample exists.
\end{definition}

Verification queries use \texttt{assume($\varphi$)} and \texttt{assert($\varphi$)} statements. \texttt{assume($\varphi$)} at line $l$ restricts traces to those satisfying $\varphi$ (equivalent to \texttt{if ($\neg\varphi$) halt}), while \texttt{assert($\varphi$)} jumps to ERROR if violated (equivalent to \texttt{if ($\neg\varphi$) goto ERROR}). For query $\langle A, P, q \rangle$, we annotate $P$ with \texttt{assume($\varphi$)} for each $(l,\varphi)\in A$ and \texttt{assert($\varphi$)} where $q=(l,\varphi)$. $P$ is safe with respect to the verification query if and only if ERROR is unreachable.

\subsection{Program Verification using Invariants}
\label{sec:verification_procedure}

Given a verification query $\langle A, P, q \rangle$, the task can be addressed through a direct invocation of the verifier: $\mathcal{V}(A, P, q)$. Alternatively, we can break the problem into sub-problems, by introducing a candidate invariant property $I = (l, \varphi)$, where $l \in \mathcal{L}(P)$. This approach splits the verification task into two distinct queries:
\begin{enumerate}
    \item \emph{Correctness Check} ($\mathcal{V}_1$): Verify that the property $I$ is correct (inductive) within the program $P$ given $A$: $\mathcal{V}_1 := \mathcal{V}(A, P, I)$
    \item \emph{Sufficiency Check} ($\mathcal{V}_2$): Verify that the target postcondition $q$ holds, assuming the correctness of the candidate invariant $I$: $\mathcal{V}_2 := \mathcal{V}(A \cup \{I\}, P, q)$
\end{enumerate}

We denote by $t_b$ the wall-clock time of direct verification $\mathcal{V}(A,P,q)$, and by $t_v = \max(t_1, t_2)$ the parallel execution time of the correctness and sufficiency checks, where $t_1$ and $t_2$ are their respective solving times.

The final outcome is determined by Table~\ref{tab:decision}. This procedure is sound; conclusive outcomes (\textsc{True} or \textsc{False}) are guaranteed to be correct \citep{wei2025invbench}. This approach augments the verification problem with new knowledge that can be useful to prove the desired property. The correctness and sufficiency checks are independent and can therefore be performed in parallel, and in the case where both queries are simpler than the original query, the wall-clock verification time is reduced.

\begin{table}[ht]
\centering
\small
\caption{Decision procedure for verification queries using a candidate invariant. Note that a \textsc{False} outcome in the sufficiency check implies the original query is invalid regardless of the invariant's correctness.}
\begin{tabular}{ccc}
\toprule
$\mathcal{V}_1$ (Correctness) & $\mathcal{V}_2$ (Sufficiency) & Outcome \\
\midrule
\textsc{True} & \textsc{True} & \textsc{True} \\
$*$ & \textsc{False} & \textsc{False} \\
\multicolumn{2}{c}{otherwise} & \textsc{Unknown} \\
\bottomrule
\end{tabular}
\label{tab:decision}
\end{table}
An illustration of the approach appears in Figure~\ref{fig:vbs}. Instead of directly verifying the query (Baseline block, top-right), we can use the invariant $I=(l, 5x+3y=300)$ and verify its correctness (Correctness Check block, bottom-left) and sufficiency (Sufficiency Check block, bottom-right). When the verifier returns $\textsc{True}$ for both of these queries, we can immediately deduce the safety of the program.

\textbf{Generating Invariants.}
The effectiveness of the aforementioned techniques is conditional upon our ability to generate useful invariant candidates, i.e., candidates that will allow us to reach a \textsc{True}/\textsc{False} outcome, as per Table~\ref{tab:decision}.
We define the invariant generation task, which is the primary focus of this work, as follows: given a verification query $\langle A, P, q \rangle$, a verification oracle $\mathcal{V}$, and a designated loop entry $l \in \mathcal{L}(P)$, the objective is to synthesize a predicate $\varphi$ such that the resulting property $I = (l, \varphi)$ satisfies the correctness and sufficiency checks; i.e., $\mathcal{V}$ returns \textsc{True} for the corresponding correctness and sufficiency queries. We also regard a correct candidate whose sufficiency check returns \textsc{False} as desirable, since this witnesses a genuine bug in $P$.

\definecolor{kw}{RGB}{160, 32, 160}     
\definecolor{op}{RGB}{180, 60, 60}      
\definecolor{num}{RGB}{20, 100, 180}    
\definecolor{id}{RGB}{40, 90, 40}       

\begin{figure*}[ht!]
\centering

\begin{tikzpicture}[
    codebox/.style={
        draw, 
        rounded corners=4pt, 
        font=\ttfamily\fontsize{7}{8.5}\selectfont, 
        align=left, 
        fill=#1, 
        inner sep=5pt
    },
    codebox/.default=white,
    llmbox/.style={
        draw, 
        rounded corners=5pt, 
        fill=violet!15, 
        font=\footnotesize\bfseries, 
        minimum height=0.6cm, 
        minimum width=1.2cm
    },
    lbl/.style={font=\fontsize{7.5}{9}\selectfont\bfseries},
    sublbl/.style={font=\tiny, text=gray},
    arrowlbl/.style={font=\tiny\itshape, fill=white, inner sep=2pt}
]

\node[codebox=orange!12, text width=4.6cm] (v0) {
(((((((((((((((((((((( \textcolor{num}{36} \textcolor{op}{<=} \textcolor{id}{y} ) \textcolor{op}{\&\&} ( \textcolor{num}{36} \textcolor{op}{<=} \textcolor{id}{x} ))\\
\quad \textcolor{op}{||} (( \textcolor{num}{21} \textcolor{op}{<=} \textcolor{id}{y} ) \textcolor{op}{\&\&} ( \textcolor{num}{45} \textcolor{op}{<=} \textcolor{id}{x} )))\\
\quad \textcolor{op}{||} (( \textcolor{num}{21} \textcolor{op}{<=} \textcolor{id}{x} ) \textcolor{op}{\&\&} ( \textcolor{num}{65} \textcolor{op}{<=} \textcolor{id}{y} )))\\
\quad \textcolor{op}{||} (( \textcolor{num}{42} \textcolor{op}{<=} \textcolor{id}{x} ) \textcolor{op}{\&\&} ( \textcolor{num}{26} \textcolor{op}{<=} \textcolor{id}{y} )))\\
\quad \textcolor{op}{||} (( \textcolor{num}{12} \textcolor{op}{<=} \textcolor{id}{x} ) \textcolor{op}{\&\&} ( \textcolor{num}{80} \textcolor{op}{<=} \textcolor{id}{y} )))\\
\quad \textcolor{op}{||} (( \textcolor{num}{11} \textcolor{op}{<=} \textcolor{id}{y} ) \textcolor{op}{\&\&} ( \textcolor{num}{51} \textcolor{op}{<=} \textcolor{id}{x} )))\\
\quad \textcolor{gray}{\vdots}\\
\quad \textcolor{op}{||} (( \textcolor{num}{16} \textcolor{op}{<=} \textcolor{id}{y} ) \textcolor{op}{\&\&} ( \textcolor{num}{48} \textcolor{op}{<=} \textcolor{id}{x} )))
};
\node[lbl, above=0.1cm of v0] {$V_0$: Raw Verifier Output};
\node[sublbl, below=0.08cm of v0] {603 chars, 21 disjuncts};

\node[codebox=yellow!12, text width=3.4cm, right=1.6cm of v0] (v1) {
\textcolor{num}{36} \textcolor{op}{<=} \textcolor{id}{y} \textcolor{op}{\&\&} \textcolor{num}{36} \textcolor{op}{<=} \textcolor{id}{x}\\
\textcolor{op}{||} \textcolor{num}{21} \textcolor{op}{<=} \textcolor{id}{y} \textcolor{op}{\&\&} \textcolor{num}{45} \textcolor{op}{<=} \textcolor{id}{x}\\
\textcolor{op}{||} \textcolor{num}{21} \textcolor{op}{<=} \textcolor{id}{x} \textcolor{op}{\&\&} \textcolor{num}{65} \textcolor{op}{<=} \textcolor{id}{y}\\
\textcolor{op}{||} \textcolor{num}{42} \textcolor{op}{<=} \textcolor{id}{x} \textcolor{op}{\&\&} \textcolor{num}{26} \textcolor{op}{<=} \textcolor{id}{y}\\
\textcolor{op}{||} \textcolor{num}{12} \textcolor{op}{<=} \textcolor{id}{x} \textcolor{op}{\&\&} \textcolor{num}{80} \textcolor{op}{<=} \textcolor{id}{y}\\
\quad \textcolor{gray}{\vdots}\\
\textcolor{op}{||} \textcolor{num}{16} \textcolor{op}{<=} \textcolor{id}{y} \textcolor{op}{\&\&} \textcolor{num}{48} \textcolor{op}{<=} \textcolor{id}{x}
};
\node[lbl, above=0.1cm of v1] {$V_1$: Normalized};
\node[sublbl, below=0.08cm of v1] {441 chars, 21 disjuncts};

\node[llmbox, right=0.6cm of v1] (llm) {LLM};

\node[codebox=green!12, right=1.2cm of llm, minimum height=1.5cm, text width=3.2cm, align=center, font=\ttfamily\small] (v2) {
\vfill
\textcolor{num}{5}\textcolor{op}{*}\textcolor{id}{x} \textcolor{op}{+} \textcolor{num}{3}\textcolor{op}{*}\textcolor{id}{y} \textcolor{op}{==} \textcolor{num}{300}
\vfill
};
\node[lbl, above=0.1cm of v2] {$V_2$: Simplified};
\node[sublbl, below=0.08cm of v2] {18 chars, 1 equation};

\draw[->, >=stealth, thick, gray] (v0.east) -- (v1.west) node[arrowlbl, midway, above] {normalize};
\draw[->, >=stealth, thick, gray] (v1.east) -- (llm.west);
\draw[->, >=stealth, thick, gray] (llm.east) -- (v2.west) node[arrowlbl, midway, above] {simplify};

\end{tikzpicture}
\caption{\textsc{Wonda} pipeline. The raw verifier output ($V_0$) is normalized ($V_1$), then an LLM simplifies it to a compact closed-form expression ($V_2$): $5x + 3y = 300$, achieving a \textbf{2x} verification speedup.}
\label{fig:pip_example1}
\end{figure*}
\section{Methodology}
\label{sec:method}
The main component of the \emph{generate-and-verify} paradigm is the LLM that performs invariant synthesis. In practice, off-the-shelf LLMs are not very effective at this task and need to be fine-tuned for the invariant generation domain. Somewhat counterintuitively, fine-tuning on data produced by modern verification tools does not show a performance boost on challenging problems~\citep{wei2025invbench}. We argue that fine-tuning can in fact be quite useful, provided the training data is of high quality. To this end, we investigate two research questions: (a) how to define high-quality training data in this context; and (b) how to produce it.

Given a verification query $\langle A, P, q \rangle$, a verification oracle $\mathcal{V}$ and loop locations $\mathcal{L}(P)$, our goal is to produce training samples that pair programs with high-quality loop invariants.
We claim that a good invariant $I=(l, \varphi)$ should satisfy the following characteristics:
\begin{enumerate}
    \item \emph{Non-Degeneracy:} we exclude trivial invariants $\varphi \in \{\textsc{False}, \textsc{True}\}$.
    \item \emph{Correctness:} invariant $\varphi$ holds at location $l$ on all executions, i.e., $\mathcal{V}(A,P,I) = \textsc{True}$.
    \item \emph{Usefulness:} using $I$ expedites verification times. This entails sufficiency, i.e., $\mathcal{V}_2 = \textsc{True}$, and that $t_v < t_b$, where $t_v$ and $t_b$ are as defined in Section~\ref{sec:verification_procedure}.
    \item \emph{Compactness:} invariant $I$ has a succinct syntactic form, which we hypothesize facilitates learning and improves generalization.
\end{enumerate}

\paragraph{Grounding to Verifier-Generated Invariants.}
Following \citet{wei2025invbench}, we initialize our curation process using invariants extracted from UAutomizer.  When UAutomizer proves a verification query $\langle \{p\}, P, q \rangle$, it emits discovered invariants $\{(l, \varphi_{\text{raw}})\}$ as part of its proof. 
These invariants satisfy \emph{correctness} by construction. However, they are not guaranteed to be \emph{useful} (they may not provide speedup) or \emph{compact} (they often contain tool-specific artifacts). Figure~\ref{fig:raw_example} depicts an example: while correct, the raw invariant is cluttered with type casts and verbose structure, making it a poor training target.

We address these issues with a two-stage pipeline:
\begin{enumerate}
    \item \emph{Invariant Normalization} (\S\ref{sec:invariant-normalization}): AST-based rewriting that removes tautologies, contradictions, minimizes parentheses and strips redundant type casts.
    \item \emph{LLM-Based Simplification} (\S\ref{sec:llm-simplification}): A data augmentation step where an LLM generalizes verbose invariants into compact candidates, each verified for correctness and usefulness.
\end{enumerate}
Figure~\ref{fig:pip_example1} illustrates output produced by our pipeline.
Another example appears in Figure~\ref{fig:pipeline_example2} (Appendix~\ref{app:pip_examples}).

\subsection{Invariant Normalization}
\label{sec:invariant-normalization}

Raw invariants from automated verifiers are often cluttered with syntactic noise that obscures underlying program logic. This noise typically manifests as: (i)~tautological clauses (e.g., \texttt{n <= n}); (ii)~redundant parenthesization, and (iii)~excessive integral type casts (e.g., \texttt{\_\_int128} and \texttt{long long}) used for bit-precise semantics. In rare cases, verifiers may even produce contradictions within unreachable or dead-code paths, such as \texttt{(x > x),} which can confuse models during training.

We apply a semantic-preserving AST normalization, \textsc{Normalize}, which performs a single bottom-up traversal using the rules in Table~\ref{tab:rewrite-rules}. It also strips the invariants of any unnecessary parentheses, based on standard C operator precedence rules.

\begin{table}[H]
\centering
\small
\caption{Rewrite rules applied during AST normalization. $e$ represents a numeric variable, $\varphi$ a predicate, $c$ a constant, and $\bowtie \in \{\leq, \geq, =, <, >, \neq\}$ a relational operator.}
\begin{tabular}{@{}ll@{}}
\toprule
\textbf{Rule} & \textbf{Transformation} \\
\midrule
\multicolumn{2}{@{}l}{\textit{Tautology Elimination}} \\
\textsc{TautConj} & $\varphi \land \mathit{true}$ or $\mathit{true} \land \varphi \;\longrightarrow\; \varphi$ \\
\textsc{TautRefl} & $e \bowtie e \;\longrightarrow\; \mathit{true}$ \quad for $\bowtie \in \{\leq, \geq, =\}$ \\
\textsc{TautConst} & $c_1 \bowtie c_2 \;\longrightarrow\; \mathit{true}$ \quad if $c_1 \bowtie c_2$ holds \\
\midrule
\multicolumn{2}{@{}l}{\textit{Contradiction Propagation}} \\
\textsc{ContraConj} & $\varphi \land \mathit{false}$ or $\mathit{false} \land \varphi \;\longrightarrow\; \mathit{false}$ \\
\textsc{ContraDisj} & $\varphi \lor \mathit{false}$ or $\mathit{false} \lor \varphi \;\longrightarrow\; \varphi$ \\
\textsc{ContraRefl} & $e \bowtie e \;\longrightarrow\; \mathit{false}$ \quad for $\bowtie \in \{<, >, \neq\}$ \\
\bottomrule
\end{tabular}
\label{tab:rewrite-rules}
\end{table}

\begin{proposition}[Normalization Soundness]
For any predicate $\varphi$, the normalized predicate $\varphi' = \textsc{Normalize}(\varphi)$ is semantically equivalent to the original ($\varphi' \equiv \varphi$).
\end{proposition}
\begin{proof}
The rewrite rules implement standard first-order logic identities applied inductively via bottom-up AST traversal. The parenthesis minimization follows the C operator precedence while preserving associativity and binding.
\end{proof}

As an additional optimization step, we eliminate integral casts to reduce clutter. While stripping casts can break soundness (e.g., under certain overflow conditions), we prioritize logical clarity for model training and formally verify the final generated invariants against the original program to ensure correctness is maintained.

\subsection{LLM-Based Invariant Simplification}
\label{sec:llm-simplification}
AST normalization removes syntactic noise, but many invariants remain verbose due to \emph{semantic} complexity: enumerated cases, redundant bounds, or overly strong constraints. These patterns require reasoning beyond local rewrites.
To bridge this gap, we employ an LLM as a \emph{simplification function} $f_\theta: \Phi \to \Phi^N$, where $\Phi$ is the set of possible predicates and $N$ is the number of generated candidates. Unlike the rules in Table~\ref{tab:rewrite-rules}, $f_\theta$ is not guaranteed to be semantics-preserving ($f_\theta(\varphi) \not\equiv \varphi$). Instead, we leverage the LLM’s ability to perform \emph{abstraction}, aiming for candidates that are more compact yet still sufficient to prove property $q$. 
The resulting predicate can be equivalent, stronger, weaker or incomparable with the original predicate; this does not jeopardize soundness, as we later invoke the verifier to ensure the correctness and sufficiency of the rewritten predicates. 

\paragraph{Simplification Procedure.} 
To guide the LLM toward high-quality invariants, we provide a prompt context $\mathcal{C} = \langle P, \varphi, \mathcal{G} \rangle$ containing the program $P$, the reference invariant $\varphi$, and a  set of transformation goals $\mathcal{G}$. The following transformation goals capture our main insights for achieving the desired characteristics of \emph{compactness} and \emph{usefulness}.

\noindent\textit{Range generalization.} We observed that a large number of raw invariants have a case-enumeration structure; for example, $\bigvee_{i=1}^k (x=i)$ might come from a loop that runs for $k$ steps. We prompt an LLM to simplify such invariants into a more compact range invariant, $1 \le x \le k$, which is easier to learn than the enumeration representation.

\noindent\textit{Constraint factoring.} Some invariants form Boolean expressions that have unnecessarily complex logical structure, e.g., $(a \land b_1) \lor \ldots \lor (a \land b_k)$. We encourage an LLM to perform logical simplification, e.g., rewrite this as $a \land (b_1 \lor \ldots \lor b_k)$.

\noindent\textit{Closed-form discovery.} Often, we would like to generalize the normalized invariant to replace case enumerations with arithmetic relations, where linear expressions are preferred for solver efficiency, though non-linear closed forms are also possible. This gives us compact, reusable invariants that are easier to learn; for example, $\sum_{i=1}^n i \to \frac{n(n+1)}{2}$.

\noindent\textit{Redundancy removal.} Some expressions are implied by program semantics and hold throughout the program, e.g., preconditions or variables with fixed values. To simplify them, we can remove these redundant clauses to obtain more compact and easier-to-learn invariants. 

Overall, our strategy encourages semantic abstraction over syntactic rewriting; for the full prompt, see Appendix~\ref{app:prompt-invariant-simplification}.

\begin{definition}[Quality Grade $G$] 
To measure the utility of a candidate $\varphi'$ with respect to a verification query $\mathcal{Q} = \langle A, P, q \rangle$ and baseline time $t_b$, we define a grading function $G(\varphi', \mathcal{Q}, t_b) \in \{0,1,2,3\}$ based on $\mathcal{V}_1$, $\mathcal{V}_2$ checks defined in \S\ref{sec:verification_procedure}.
\[
G(\varphi', \mathcal{Q}, t_b) = \begin{cases} 
0 &  \mathcal{V}_1 \neq \textsc{True} \\
1 &  \mathcal{V}_1 = \textsc{True} \land \mathcal{V}_2 \neq \textsc{True} \\
2 &  \mathcal{V}_1 = \textsc{True} \land \mathcal{V}_2 = \textsc{True} \land t_v \ge t_b \\
3 &  \mathcal{V}_1 = \textsc{True} \land \mathcal{V}_2 = \textsc{True} \land t_v < t_b
\end{cases}
\]
where $t_v = \max(t_1, t_2)$ and $t_b$ are as defined in Section~\ref{sec:verification_procedure}. The full grading procedure is given in Algorithm~\ref{alg:verify} (Appendix~\ref{app:algorithms}). Syntactically invalid candidates are discarded before grading. We retain candidates with $G(\varphi', \mathcal{Q}, t_b) \ge 2$ as ``golden'' training samples.
\end{definition}

The simplification procedure is given in Algorithm~\ref{alg:simplify} (Appendix~\ref{app:algorithms}). It first discards degenerate invariants, i.e., $\varphi_{\text{norm}} \in \{\textsc{False}, \textsc{True}\}$. Otherwise, if the invariant is deemed verbose (in our implementation, if $|\varphi_{\text{norm}}| > \eta$ for a threshold $\eta>0$), an LLM generates $N$ candidates. These candidates are deduplicated by exact match, filtered again for degeneracy, and each is graded using Algorithm~\ref{alg:verify}; we retain those with grade $g \ge 2$ as pairs $(\varphi, g)$. If no candidate qualifies (or if $\varphi_{\text{norm}}$ is not verbose), we instead grade $\varphi_{\text{norm}}$ itself.

\begin{proposition}[Pipeline Soundness]
If the pipeline outputs $(l, \varphi^*)$, then $\varphi^*$ is a correct loop invariant sufficient to prove $q$.
\end{proposition}

\begin{proof}
By construction, $\varphi^*$ is returned only if it passes the formal verifier's correctness check ($\mathcal{V}_1 = \textsc{True}$) and sufficiency check ($\mathcal{V}_2 = \textsc{True}$).
\end{proof}
\section{Experimental Setup}

\textbf{Training Dataset.}
We fine-tune on raw verifier-generated invariants extracted from training programs released by~\citet{wei2025invbench}.
We run UAutomizer to verify each program and collect loop invariants from its output; we denote this raw collection~\textit{V0}.
AST-based normalization yields~\textit{V1}, and LLM-driven simplification (Kimi K2 Thinking~\citep{team2025kimi}, $N{=}4$ candidates per verbose invariant, where $\eta=20$ characters) followed by verifier filtering yields~\textit{V2}.
Retaining only $g\geq 2$ candidates for fine-tuning yields 7{,}284 samples, partitioned 80/20 into train and validation.
Full pipeline yield and dataset statistics are in Appendix~\ref{app:training-data-statistics}.

\newpage
\textbf{Evaluation Dataset.} For evaluation, we use InvBench, a benchmark of C verification queries derived from SV-COMP~\citep{beyer2025improvements}, released by~\citet{wei2025invbench}. Starting from 219 programs, we expand 
multi-loop programs into per-loop instances, yielding 362 total. We partition these into  \emph{Easy} ($n=239$, 66\%) and \emph{Hard} ($n=123$, 34\%) using a 15-second 
UAutomizer baseline threshold, focusing on Hard instances where invariants provide 
meaningful speedup; 20 Hard instances time out (Appendix~\ref{app:baseline-hard-cases}).  Each instance in $\mathcal{D} = \{(p_i, P_i, q_i, l_i, t_b^{(i)})\}_{i=1}^n$  comprises a precondition $p_i$, program $P_i$, postcondition $q_i$, loop location  $l_i \in \mathcal{L}(P_i)$, and median baseline time $t_b^{(i)}$ over $k=3$ runs. Models are prompted with a structured JSON format (Appendix~\ref{app:prompt-eval}); syntactically invalid, ill-formed, or outputs using side-effect operators (e.g., $++, +=, =$) receive verdict \textsc{Unknown}.
UAutomizer configuration, and hardware used for evaluation are detailed in Appendix~\ref{app:uautomizer_config}.

\textbf{Models.} We evaluate three open model families. From \emph{Qwen3}~\citep{Yang2025Qwen3TR} we use Qwen3-0.6B, Qwen3-4B-Instruct-2507, Qwen3-8B, and Qwen3-14B (referred to as \emph{Qwen3-0.6B/4B/8B/14B}); we also use \emph{Llama-3.1-8B-Instruct}~\citep{grattafiori2024llama3} and \emph{Mistral-7B-Instruct-v0.3}~\citep{jiang2023mistral}, referred to as \emph{Llama-3.1-8B} and \emph{Mistral-7B}. All models are fully fine-tuned except Qwen3-8B, which uses LoRA~\citep{Hu2022LoRALA} (Appendix~\ref{app:training-details}). As off-the-shelf baselines, we compare against Qwen3-Next-80B-A3B-Instruct (referred to as \emph{Qwen3-80B}), GPT-OSS-120B~\citep{agarwal2025gpt}, and GPT-5.2~\citep{openai2025gpt52}. The Qwen3-0.6B, 8B, and 14B models are trained and evaluated in \emph{non-thinking mode}.

\textbf{Decision Procedure.}
For valid candidates, we apply the procedure from \S\ref{sec:verification_procedure}, executing two queries \emph{in parallel}:
\begin{align*}
\mathcal{V}_1^{(i)} &= \mathcal{V}(\{p_i\},\, P_i,\, (l_i, \hat{\varphi}_i)) && \text{(Correctness Check)} \\
\mathcal{V}_2^{(i)} &= \mathcal{V}(\{p_i, (l_i,\, \hat{\varphi}_i)\}, P_{i}, q_i) && \text{(Sufficiency Check)}
\end{align*}
with wall-clock times $t_1^{(i)}$ and $t_2^{(i)}$, respectively.
Since both queries execute in parallel, the verification time is $t_v^{(i)} = \max(t_1^{(i)}, t_2^{(i)})$.
The outcome $D_i$ is determined by Table~\ref{tab:decision}.

\textbf{Evaluation Metrics.}
We define binary indicators for each instance $i$:
\begin{enumerate}
    \item $\textsc{Valid}(i) = 1$ iff syntactic validation passes;
    \item $\textsc{Correct}(i) = 1$ iff $\textsc{Valid}(i)$ and $\mathcal{V}_1^{(i)} = \textsc{True}$;
    \item $\textsc{Speedup}(i) = 1$ iff $\textsc{Correct}(i)$,\\  $D_i \in \{\textsc{True}, \textsc{False}\}$, and $t_v^{(i)} < t_b^{(i)}$.
\end{enumerate}

We report the mean of each indicator across the evaluation set. We also define the speedup factor:
\[
S(i) = \begin{cases} 
t_b^{(i)} / t_v^{(i)} & \text{if } \textsc{Correct}(i)~\land \\
& \quad D_i \in \{\textsc{True}, \textsc{False}\} \\ 
1 & \text{otherwise} 
\end{cases}
\]

We report $\bar{S}_{>1}$ as the mean speedup factor among instances with $\textsc{Speedup}(i) = 1$.

\textbf{Virtual Best Solver.}
The verification procedure described in Section~\ref{sec:verification_procedure} is designed to run as part of a portfolio, in parallel to direct verification: the LLM proposes an invariant while the baseline verifier runs concurrently, and whichever finishes first determines the outcome. The \emph{Virtual Best Solver (VBS)} metric captures this by selecting the faster strategy per instance (cf.\ Figure~\ref{fig:vbs}):
\[
\text{VBS}(i) = \begin{cases} 
\min(t_v^{(i)}, t_b^{(i)}) & \text{if } \textsc{Correct}(i)~\land \\ 
&\quad D_i \in \{\textsc{True}, \textsc{False}\}\\
t_b^{(i)} & \text{otherwise} 
\end{cases}
\]

The \emph{Virtual Best Performance} $\text{VBP} = \frac{1}{n}\sum_i \text{VBS}(i)$ measures the average verification time achievable by optimally combining LLM invariants with baseline direct verification.

\textbf{Note on Model Latency.}
Our primary metrics ($R_{\mathrm{speedup}}$, $\bar{S}_{>1}$, and VBP) exclude inference latency $t_m^{(i)}$ to isolate model capability from deployment factors. For end-to-end cost, we additionally report $\text{VBP}_{\text{E2E}}$, which adds $t_m^{(i)}$ to $t_v^{(i)}$ for sufficient instances.
: $\text{VBS}_{\text{E2E}}(i) = \min(t_v^{(i)} + t_m^{(i)},\, t_b^{(i)})$.

\begin{table*}[ht]
\centering
\small
\setlength{\tabcolsep}{0.55em}
\caption{Main results on the Hard instances ($n=123$). Results shown as mean $\pm$ std.\ across three runs. $R_{\text{valid/correct/speedup}}$: indicator rates (\%). $\bar{S}_{>1}$: mean speedup among accelerated instances. VBP: Virtual Best Performance in seconds (verifier-only baseline VBP: 193s). Solved: baseline timeouts (of 20) resolved per run. \textbf{Bold}: indicates best per model scale.}
\label{tab:results-main}
\begin{tabular}{l | r r r r | r r r}
\toprule
\emph{Model} & $R_{\text{valid}}\ (\%)$ & $R_{\text{correct}}\ (\%)$ & $R_{\text{speedup}}\ (\%)$ & $\bar{S}_{>1}$\ (x) & VBP $\downarrow$\ (s) & $\text{VBP}_{\text{E2E}}$ $\downarrow$\ (s) & Solved \\
\midrule
GPT-5.2         & 94.0$\pm$1.7 & 72.4$\pm$2.2 & 37.1$\pm$1.2 & 10.7$\pm$0.4 & 155.6$\pm$3.0 & 163.4$\pm$3.0 & 3, 2, 3 \\
GPT-OSS-120B    & 92.1$\pm$1.2 & 58.0$\pm$1.2 & 27.4$\pm$2.9 &  7.0$\pm$1.4 & 165.8$\pm$5.6 & 167.6$\pm$5.7 & 3, 2, 1 \\
Qwen3-80B       & 97.8$\pm$0.5 & 38.8$\pm$1.7 & 21.4$\pm$1.2 &  9.5$\pm$1.0 & 169.5$\pm$2.1 & 169.7$\pm$2.1 & 4, 3, 3 \\
\midrule
Qwen3-14B       & 96.5$\pm$0.5 & 36.3$\pm$1.9 & 13.6$\pm$2.0 &  7.6$\pm$0.9 & 183.2$\pm$1.2 & 183.6$\pm$1.1 & 1, 2, 1 \\
Qwen3-14B-V2 (Ours) & \textbf{100.0$\pm$0.0} & \textbf{43.4$\pm$4.9} & \textbf{18.4$\pm$4.2} & \textbf{16.0$\pm$3.3} & \textbf{162.1$\pm$8.3} & \textbf{162.9$\pm$8.1} & \textbf{4, 4, 2} \\
\midrule
Qwen3-8B (Base) & 89.4$\pm$7.8 & 23.8$\pm$3.1 & 10.8$\pm$0.5 &  8.5$\pm$5.2 & 181.6$\pm$4.3 & 181.7$\pm$4.2 & 0, 0, 3 \\
Qwen3-8B-V2 (Ours) & \textbf{100.0$\pm$0.0} & \textbf{42.8$\pm$4.6} & \textbf{21.7$\pm$1.7} & \textbf{10.7$\pm$2.3} & \textbf{166.5$\pm$4.3} & \textbf{166.7$\pm$4.3} & \textbf{2, 1, 4} \\
\midrule
Qwen3-4B (Base) & 99.2$\pm$0.0 & 22.8$\pm$2.2 & 11.1$\pm$0.9 &  8.9$\pm$2.5 & 185.6$\pm$2.4 & 185.7$\pm$2.4 & 1, 0, 1 \\
Qwen3-4B-V2 (Ours) & \textbf{100.0$\pm$0.0} & \textbf{44.4$\pm$2.3} & \textbf{24.7$\pm$1.2} & \textbf{12.4$\pm$2.2} & \textbf{165.5$\pm$3.2} & \textbf{165.7$\pm$3.2} & \textbf{3, 2, 2} \\
\midrule
Qwen3-0.6B (Base) & 88.3$\pm$0.5 & \textbf{28.5$\pm$2.8} & 12.2$\pm$2.2 &  5.3$\pm$3.3 & 182.9$\pm$5.7 & 183.0$\pm$5.7 & 2, 0, 1 \\
Qwen3-0.6B-V2 (Ours) & \textbf{99.7$\pm$0.5} & 27.9$\pm$0.5 & \textbf{14.1$\pm$2.5} & \textbf{8.5$\pm$3.1} & \textbf{174.0$\pm$5.6} & \textbf{174.1$\pm$5.6} & \textbf{2, 2, 1} \\
\midrule
Llama3.1-8B (Base) & 96.2$\pm$1.2 & 14.9$\pm$1.7 &  5.4$\pm$2.0 &  3.7$\pm$1.4 & 186.3$\pm$5.1 & 186.4$\pm$5.1 & 1, 0, 2 \\
Llama3.1-8B-V2 (Ours) & \textbf{99.7$\pm$0.5} & \textbf{45.5$\pm$1.4} & \textbf{18.4$\pm$2.5} & \textbf{11.3$\pm$2.1} & \textbf{168.9$\pm$7.6} & \textbf{169.2$\pm$7.5} & \textbf{3, 2, 4} \\
\midrule
Mistral-7B (Base) & 93.8$\pm$0.9 & 18.4$\pm$2.0 &  7.3$\pm$1.4 &  6.7$\pm$3.6 & 186.7$\pm$1.5 & 186.9$\pm$1.5 & 0, 0, 0 \\
Mistral-7B-V2 (Ours) & \textbf{99.5$\pm$0.9} & \textbf{31.4$\pm$0.5} & \textbf{16.0$\pm$0.5} & \textbf{17.0$\pm$0.9} & \textbf{169.0$\pm$4.3} & \textbf{169.4$\pm$4.3} & \textbf{2, 1, 4} \\
\bottomrule
\end{tabular}
\end{table*}

\begin{table*}[ht]
\centering
\small
\setlength{\tabcolsep}{0.55em}
\caption{\textsc{Wonda} ablation study on the hard split ($n{=}123$; mean $\pm$ std.\ over three runs).
\emph{V0}: raw UAutomizer invariants; \emph{V1}: AST-normalized; \emph{V2}: full pipeline.
\textbf{Bold}: best per model family.}
\label{tab:ablation-mega}
\begin{tabular}{l | r r r r | r r}
\toprule
\emph{Model} & $R_{\text{valid}}\ (\%)$ & $R_{\text{correct}}\ (\%)$ & $R_{\text{speedup}}\ (\%)$ & $\bar{S}_{>1}$\ (x) & VBP $\downarrow$\ (s) & $\text{VBP}_{\text{E2E}}$ $\downarrow$\ (s) \\
\midrule
Qwen3-8B (Base) & 89.4$\pm$7.8 & 23.8$\pm$3.1 & 10.8$\pm$0.5 & 8.5$\pm$5.2 & 181.6$\pm$4.3 & 181.7$\pm$4.2 \\
Qwen3-8B-V0 & 88.1$\pm$3.9 & 29.9$\pm$3.4 & 11.5$\pm$1.9 & 9.4$\pm$2.6 & 180.0$\pm$2.7 & 180.6$\pm$2.7 \\
Qwen3-8B-V1 & 97.0$\pm$0.9 & 30.1$\pm$0.8 & 13.0$\pm$2.2 & 9.1$\pm$1.9 & 175.3$\pm$3.2 & 175.5$\pm$3.2 \\
Qwen3-8B-V2 (Ours) & \textbf{100.0$\pm$0.0} & \textbf{42.8$\pm$4.6} & \textbf{21.7$\pm$1.7} & \textbf{10.7$\pm$2.3} & \textbf{166.5$\pm$4.3} & \textbf{166.7$\pm$4.3} \\
\midrule
Qwen3-4B (Base) & 99.2$\pm$0.0 & 22.8$\pm$2.2 & 11.1$\pm$0.9 & 8.9$\pm$2.5 & 185.6$\pm$2.4 & 185.7$\pm$2.4 \\
Qwen3-4B-V0 & 81.3$\pm$0.8 & 29.3$\pm$3.5 & 13.6$\pm$1.9 & 10.1$\pm$1.5 & 177.5$\pm$1.5 & 177.7$\pm$1.5 \\
Qwen3-4B-V1 & 97.6$\pm$1.4 & 33.1$\pm$2.3 & 12.7$\pm$2.3 & 11.4$\pm$2.9 & 174.2$\pm$4.7 & 174.4$\pm$4.7 \\
Qwen3-4B-V2 (Ours) & \textbf{100.0$\pm$0.0} & \textbf{44.4$\pm$2.3} & \textbf{24.7$\pm$1.2} & \textbf{12.4$\pm$2.2} & \textbf{165.5$\pm$3.2} & \textbf{165.7$\pm$3.2} \\
\midrule
Qwen3-0.6B (Base) & 88.3$\pm$0.5 & \textbf{28.5$\pm$2.8} & 12.2$\pm$2.2 & 5.3$\pm$3.3 & 182.9$\pm$5.7 & 183.0$\pm$5.7 \\
Qwen3-0.6B-V0 & 85.9$\pm$2.6 & 18.7$\pm$0.8 & 8.9$\pm$0.8 & 11.7$\pm$9.4 & 178.0$\pm$2.7 & 178.1$\pm$2.7 \\
Qwen3-0.6B-V1 & 97.6$\pm$1.4 & 23.3$\pm$1.2 & 9.8$\pm$2.2 & \textbf{15.3$\pm$0.9} & 174.4$\pm$4.4 & 174.5$\pm$4.4 \\
Qwen3-0.6B-V2 (Ours) & \textbf{99.7$\pm$0.5} & 27.9$\pm$0.5 & \textbf{14.1$\pm$2.5} & 8.5$\pm$3.1 & \textbf{174.0$\pm$5.6} & \textbf{174.1$\pm$5.6} \\
\midrule
Llama3.1-8B (Base) & 96.2$\pm$1.2 & 14.9$\pm$1.7 & 5.4$\pm$2.0 & 3.7$\pm$1.4 & 186.3$\pm$5.1 & 186.4$\pm$5.1 \\
Llama3.1-8B-V0 & 88.9$\pm$0.5 & 31.2$\pm$2.0 & 14.9$\pm$0.9 & 9.3$\pm$3.1 & 175.0$\pm$2.5 & 175.4$\pm$2.5 \\
Llama3.1-8B-V1 & 99.7$\pm$0.5 & 36.3$\pm$5.4 & 14.1$\pm$3.7 & \textbf{15.3$\pm$3.4} & 170.0$\pm$3.5 & 170.4$\pm$3.5 \\
Llama3.1-8B-V2 (Ours) & \textbf{99.7$\pm$0.5} & \textbf{45.5$\pm$1.4} & \textbf{18.4$\pm$2.5} & 11.3$\pm$2.1 & \textbf{168.9$\pm$7.6} & \textbf{169.2$\pm$7.5} \\
\midrule
Mistral-7B (Base)      & 93.8$\pm$0.9 & 18.4$\pm$2.0 & 7.3$\pm$1.4 & 6.7$\pm$3.6 & 186.7$\pm$1.5 & 186.9$\pm$1.5 \\
Mistral-7B-V0          & 68.8$\pm$2.0 & 17.1$\pm$4.3 & 6.5$\pm$2.8 & 16.8$\pm$2.3 & 179.1$\pm$7.7 & 179.4$\pm$7.7 \\
Mistral-7B-V1          & 95.9$\pm$1.6 & 24.9$\pm$0.9 & 10.0$\pm$3.8 & 14.3$\pm$2.8 & 175.4$\pm$6.4 & 175.9$\pm$6.2 \\
Mistral-7B-V2 (Ours)   & \textbf{99.5$\pm$0.9} & \textbf{31.4$\pm$0.5} & \textbf{16.0$\pm$0.5} & \textbf{17.0$\pm$0.9} & \textbf{169.0$\pm$4.3} & \textbf{169.4$\pm$4.3} \\
\bottomrule
\end{tabular}
\end{table*}

\section{Results and Analysis}
\paragraph{Benefits of \textsc{Wonda}.}

Table~\ref{tab:results-main} depicts our main results on the hard split. Across all open-model scales, V2 curation significantly boosts performance relative to the corresponding base. \emph{Qwen3-4B-V2} more than doubles its base model's speedup rate and nearly doubles its correctness, with $\mathrm{VBP}_{\mathrm{E2E}}$ falling from $185.7\,s$ to $165.7\,s$. \emph{Qwen3-8B-V2} shows the same pattern, \emph{Llama-3.1-8B-V2} triples base correctness ($14.9\% \to 45.5\%$), and \emph{Mistral-7B-V2} more than doubles its speedup rate ($7.3\% \to 16.0\%$). \emph{Qwen3-14B-V2} achieves the best VBP among our models ($162.1\,s$), matching GPT-5.2 on $\mathrm{VBP}_{\mathrm{E2E}}$ ($162.9\,s$ vs.\ $163.4\,s$). VBP improves by up to 16.0\% (${\sim}31\,s$) over direct verification alone across all families. All fine-tuned Qwen3 models at 4B parameters and above, as well as \emph{Llama-3.1-8B-V2}, outperform the off-the-shelf \emph{Qwen3-80B} on both correctness and $\mathrm{VBP}_{\mathrm{E2E}}$, with \emph{Qwen3-4B-V2} the most notable given it is ${\sim}20\times$ smaller; it also matches GPT-OSS-120B on $\mathrm{VBP}_{\mathrm{E2E}}$ ($165.7\,s$ vs.\ $167.6\,s$). Figure~\ref{fig:pareto-linked-e2e-accessible} summarizes these comparisons visually.

\begin{figure}[H]
  \centering
  \includegraphics[width=\columnwidth]{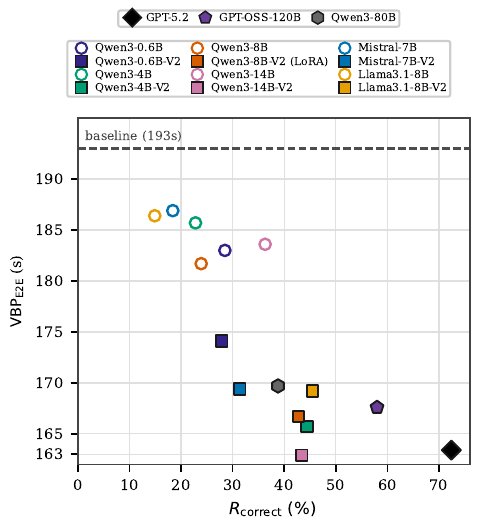}
  \caption{%
    \textbf{Correctness vs.\ $\mathrm{VBP}_{\mathrm{E2E}}$} on hard split ($n{=}123$; mean over three runs).
    Dashed line: baseline VBP (193\,s); better is bottom-right.
    V2 improves both axes across most of the models; \emph{Qwen3-4/8B-V2} match \emph{GPT-OSS-120B} on $\mathrm{VBP}_{\mathrm{E2E}}$, with \emph{Qwen3-14B-V2} matching \emph{GPT-5.2}, despite lower correctness rate.
  }
  \label{fig:pareto-linked-e2e-accessible}
\end{figure}

\paragraph{Invariant Correctness vs.\ Verification Speedup.}
Figure~\ref{fig:pareto-linked-e2e-accessible} summarizes the joint improvement in correctness and $\mathrm{VBP}_{\mathrm{E2E}}$ across families: in most cases, V2 (squares) improves over base (circles) on both axes.
However, a correct invariant need not accelerate verification; Table~\ref{tab:results-main} separates $R_{\mathrm{correct}}$ from $R_{\mathrm{speedup}}$ and $\mathrm{VBP}_{\mathrm{E2E}}$ for this reason.
\emph{Qwen3-0.6B} illustrates the decoupling most clearly: base and V2 achieve similar $R_{\mathrm{correct}}$ ($28.5\%$ vs.\ $27.9\%$), yet V2 yields higher $R_{\mathrm{speedup}}$ ($14.1\%$ vs.\ $12.2\%$), larger $\bar{S}_{>1}$ ($8.5\times$ vs.\ $5.3\times$), and lower $\mathrm{VBP}_{\mathrm{E2E}}$ ($174.1\,s$ vs.\ $183.0\,s$). \textsc{Wonda} targets this gap by curating invariants that are not only correct but also practically beneficial to the verifier.

\paragraph{Timeouts remain a challenge.}
Resolving baseline timeouts remains an open problem: even the strongest models clear only a small fraction of the 20 hard-split instances that timed out (Table~\ref{tab:results-main}, \emph{Solved}). Still, V2 makes measurable progress. \emph{Qwen3-14B-V2} averages ${\sim}3.3$ resolved per run, the highest among all models and above GPT-5.2 (${\sim}2.7$); \emph{Llama-3.1-8B-V2} also exceeds GPT-5.2 at ${\sim}3.0$. \emph{Mistral-7B} base stays at $0$, while \emph{Mistral-7B-V2} averages ${\sim}2.3$ (up to $4$). The consistent base$\to$V2 improvement across families suggests that \textsc{Wonda} teaches invariants that help the verifier resolve previously intractable instances.

\textbf{Ablation Analysis.}
Table~\ref{tab:ablation-mega} traces \textit{V0}$\to$\textit{V1}$\to$\textit{V2} on Qwen3-4B/8B, Llama-3.1-8B, and Mistral-7B.
\textit{V0} fine-tuning on raw verifier-generated invariants often \emph{lowers} $R_{\mathrm{valid}}$ (e.g., Qwen3-4B $99.2\% \to 81.3\%$; Mistral-7B $93.8\% \to 65.0\%$); \textit{V1} normalization restores syntax but yields only modest gains. \textit{V2} brings the largest improvements across all families: correctness and speedup rates nearly double on Qwen3-4B and 8B, Llama-3.1-8B-V2 reaches the highest $R_{\mathrm{correct}}$ in the table ($45.5\%$), and Mistral-7B recovers from the V0 drop. Only the full \textsc{Wonda} pipeline consistently delivers both syntactic reliability and meaningful verification speedup.
This trend is illustrated concretely in Figure~\ref{fig:inference_outputs_full} in Appendix~\ref{app:pip_examples} on a concrete example, where UAutomizer itself, the base model, and the V0/V1 stages produce verbose or incorrect invariants, while V2 discovers a compact, correct, and sufficient invariant yielding a $39.75\times$ end-to-end speedup.

\begin{table}[ht]
\centering
\small
\setlength{\tabcolsep}{0.8em}
\caption{Ablation results on the Easy Split ($n = 239$). \textbf{Bold}: best per model family.}

\label{tab:ablation-easy}
\begin{tabular}{l | r r}
\toprule
\emph{Model} & $R_{\text{valid}}\ (\%)$ & $R_{\text{correct}}\ (\%)$ \\
\midrule
Qwen3-4B (Base) & 100.0$\pm$0.0 & 42.3$\pm$1.1 \\
Qwen3-4B-V0 & 85.1$\pm$2.5 & 26.9$\pm$3.8 \\
Qwen3-4B-V1 & 99.0$\pm$0.6 & 34.5$\pm$1.5 \\
Qwen3-4B-V2 (Ours) & \textbf{100.0$\pm$0.0} & \textbf{50.8$\pm$0.9} \\
\midrule
Mistral-7B (Base) & 96.5$\pm$0.6 & 25.8$\pm$3.1 \\
Mistral-7B-V0 & 78.0$\pm$1.7 & 19.5$\pm$0.4 \\
Mistral-7B-V1 & 98.2$\pm$0.2 & 34.6$\pm$1.3 \\
Mistral-7B-V2 (Ours) & \textbf{99.7$\pm$0.2} & \textbf{41.7$\pm$1.7} \\
\midrule
Llama3.1-8B (Base) & 99.2$\pm$0.4 & 13.4$\pm$0.4 \\
Llama3.1-8B-V0 & 91.1$\pm$1.5 & 31.2$\pm$2.7 \\
Llama3.1-8B-V1 & 99.7$\pm$0.2 & 41.3$\pm$1.7 \\
Llama3.1-8B-V2 (Ours) & \textbf{99.7$\pm$0.2} & \textbf{57.6$\pm$2.1} \\
\bottomrule
\end{tabular}
\end{table}

\paragraph{Easy split.}
Table~\ref{tab:ablation-easy} shows easy-split results. V2 lifts correctness across all families, with \emph{Llama-3.1-8B-V2} reaching $57.6\%$ and \emph{Qwen3-4B-V2} reaching $50.8\%$. The V0 drop and V1 partial recovery pattern is consistent across all three families, matching the hard split. We report only validation and correctness here as the average direct verification time is already very short (${\sim}6.15\,s$) on this split.

\section{Limitations \& Future Work}
\textsc{Wonda} currently relies on UAutomizer as its sole source of training invariants, limiting coverage to solver-verifiable programs and risking inheriting the solver's biases. Extending \textsc{Wonda} to additional backend verifiers is a natural next step. Beyond the training data itself, our current setup neither exposes intermediate reasoning steps nor leverages verifier feedback. Promising directions include chain-of-thought supervision, reinforcement learning driven by verification outcomes, and integrating \textsc{Wonda}-trained models into iterative, counterexample-guided refinement loops.
\section{Conclusion}
We introduced \textsc{Wonda}, a novel data curation pipeline that transforms raw verifier-generated invariants into compact, high-quality training signals with provable quality guarantees. Our results show that data quality, not model scale alone, is a key bottleneck for neural invariant generation: fine-tuning small models on \textsc{Wonda}-curated data substantially improves both invariant correctness and verification speedup across the Qwen3, Llama-3.1, and Mistral families. Most notably, a 4B model surpasses a 20$\times$ larger model, and our best 14B model matches frontier models such as GPT-5.2 on end-to-end verification time, without test-time compute overhead. These findings suggest that careful data curation is a practical path to integrating small language models into traditional verifiers, making program verification faster and more accessible.

\section*{Impact Statement}
This work contributes to our understanding of how to curate training data for improving model performance on logical reasoning tasks. In addition, this work helps make formal verification more practical in high-stakes domains where the reliability of the software system is important. We do not anticipate significant negative societal impacts arising from this work.
\section*{Acknowledgments}
The work of Pinto, Elboher and Katz was partially funded by the European Union (RobustifAI project, ID 101212818). Views and opinions expressed are however those of the author(s) only and do not necessarily reflect those of the European Union or the European Health and Digital Executive Agency (HADEA). Neither the European Union nor the granting authority can be held responsible for them. The work of Wu is partially supported by a gift from the VMware University Research Fund.
\input{tikz_figures/standalone}
\bibliographystyle{style/icml2026}
\bibliography{refs}

\onecolumn
\appendix
\section*{Appendix} 

\section{Additional Examples}\label{app:pip_examples}
\input{tikz_figures/eval_figure}
\definecolor{kw}{RGB}{160, 32, 160}     
\definecolor{op}{RGB}{180, 60, 60}      
\definecolor{num}{RGB}{20, 100, 180}    
\definecolor{id}{RGB}{40, 90, 40}       

\begin{figure}[H]
\centering
\begin{tikzpicture}[
    codebox/.style={draw, rounded corners=4pt, font=\ttfamily\tiny, align=left, fill=#1, inner sep=5pt},
    codebox/.default=white,
    llmbox/.style={draw, rounded corners=5pt, fill=violet!15, font=\scriptsize\bfseries, minimum height=0.6cm},
    lbl/.style={font=\scriptsize\bfseries},
    sublbl/.style={font=\tiny, text=gray},
    arrowlbl/.style={font=\tiny\itshape, fill=white, inner sep=2pt},
]

\node[codebox=orange!12, text width=4.2cm] (v0) {
    ((((((((\textcolor{num}{7} \textcolor{op}{<=} \textcolor{id}{i}) \textcolor{op}{\&\&} (\textcolor{id}{N} \textcolor{op}{<=} \textcolor{num}{10}))\\
    \quad \textcolor{op}{||} ((\textcolor{num}{2} \textcolor{op}{==} \textcolor{id}{i}) \textcolor{op}{\&\&} (\textcolor{id}{N} \textcolor{op}{<=} \textcolor{num}{10})))\\
    \quad \textcolor{op}{||} ((\textcolor{num}{5} \textcolor{op}{==} \textcolor{id}{i}) \textcolor{op}{\&\&} (\textcolor{id}{N} \textcolor{op}{<=} \textcolor{num}{10})))\\
    \quad \textcolor{op}{||} ((\textcolor{id}{i} \textcolor{op}{==} \textcolor{num}{1}) \textcolor{op}{\&\&} (\textcolor{id}{N} \textcolor{op}{<=} \textcolor{num}{10})))\\
    \quad \textcolor{op}{||} ((\textcolor{num}{3} \textcolor{op}{==} \textcolor{id}{i}) \textcolor{op}{\&\&} (\textcolor{id}{N} \textcolor{op}{<=} \textcolor{num}{10})))\\
    \quad \textcolor{op}{||} ((\textcolor{id}{i} \textcolor{op}{==} \textcolor{num}{4}) \textcolor{op}{\&\&} (\textcolor{id}{N} \textcolor{op}{<=} \textcolor{num}{10})))\\
    \quad \textcolor{op}{||} ((\textcolor{num}{6} \textcolor{op}{<=} \textcolor{id}{i}) \textcolor{op}{\&\&} (\textcolor{id}{N} \textcolor{op}{<=} \textcolor{num}{10})))
};
\node[lbl, above=0.1cm of v0] {$V_0$: Raw Verifier Output};
\node[sublbl, below=0.1cm of v0] {239 chars, 7 disjuncts};

\node[codebox=yellow!12, text width=3.8cm, right=1.2cm of v0] (v1) {
    \textcolor{num}{7} \textcolor{op}{<=} \textcolor{id}{i} \textcolor{op}{\&\&} \textcolor{id}{N} \textcolor{op}{<=} \textcolor{num}{10}\\
    \textcolor{op}{||} \textcolor{num}{2} \textcolor{op}{==} \textcolor{id}{i} \textcolor{op}{\&\&} \textcolor{id}{N} \textcolor{op}{<=} \textcolor{num}{10}\\
    \textcolor{op}{||} \textcolor{num}{5} \textcolor{op}{==} \textcolor{id}{i} \textcolor{op}{\&\&} \textcolor{id}{N} \textcolor{op}{<=} \textcolor{num}{10}\\
    \textcolor{op}{||} \textcolor{id}{i} \textcolor{op}{==} \textcolor{num}{1} \textcolor{op}{\&\&} \textcolor{id}{N} \textcolor{op}{<=} \textcolor{num}{10}\\
    \textcolor{op}{||} \textcolor{num}{3} \textcolor{op}{==} \textcolor{id}{i} \textcolor{op}{\&\&} \textcolor{id}{N} \textcolor{op}{<=} \textcolor{num}{10}\\
    \textcolor{op}{||} \textcolor{id}{i} \textcolor{op}{==} \textcolor{num}{4} \textcolor{op}{\&\&} \textcolor{id}{N} \textcolor{op}{<=} \textcolor{num}{10}\\
    \textcolor{op}{||} \textcolor{num}{6} \textcolor{op}{<=} \textcolor{id}{i} \textcolor{op}{\&\&} \textcolor{id}{N} \textcolor{op}{<=} \textcolor{num}{10}
};
\node[lbl, above=0.1cm of v1] {$V_1$: Normalized};
\node[sublbl, below=0.1cm of v1] {147 chars, 7 disjuncts};

\node[llmbox, right=1.2cm of v1] (llm) {LLM};

\node[codebox=green!12, right=1.2cm of llm, minimum height=1.5cm, text width=3.2cm, align=center, font=\ttfamily\small] (v2) {
    \textcolor{num}{1} \textcolor{op}{<=} \textcolor{id}{i} \textcolor{op}{\&\&} \textcolor{id}{i} \textcolor{op}{<=} \textcolor{id}{N} \textcolor{op}{\&\&} \textcolor{id}{N} \textcolor{op}{<=} \textcolor{num}{10}
};
\node[lbl, above=0.1cm of v2] {$V_2$: Simplified};
\node[sublbl, below=0.1cm of v2] {30 chars, Grade 3};

\draw[->,  >=stealth, thick, gray] (v0.east) -- (v1.west) node[arrowlbl, midway, above] {normalize};
\draw[->,   >=stealth, thick, gray] (v1.east) -- (llm.west);
\draw[->,  >=stealth, thick, gray] (llm.east) -- (v2.west) node[arrowlbl, midway, above] {simplify};

\end{tikzpicture}
\caption{\textsc{Wonda} pipeline. The LLM factors out the common constraint \texttt{N <= 10} and reduces case enumeration into a compact range, achieving a \textbf{3.47x} verification speedup.}
\label{fig:pipeline_example2}
\end{figure}
\newpage
\section{Prompts}
\label{app:prompt}
\subsection{Prompt for training and evaluation}
\label{app:prompt-eval}
We provide here the system and user prompts for loop invariant generation task, used for both training and evaluation.
\begin{lstlisting}[
    basicstyle=\ttfamily\small, % Typewriter font, small size
    breaklines=true,            % TELLS LATEX TO WRAP LONG LINES
    breakatwhitespace=true,     % Only break at spaces, not mid-word
    columns=fullflexible,       % Tightens character spacing
    frame=single,               % Adds a box around the prompt
    rulecolor=\color{black!30}, % Light gray frame border
    backgroundcolor=\color{gray!5}, % Very light gray background
    showstringspaces=false      % Don't show spaces as underscores
]
System:
    You are an expert C programmer and highly proficient 
    in generating strong loop invariants
    for C programs that accelerate traditional verifiers' verification process.

    ## Input format
    - A C program instrumented with loop markers of the form: 
      ```c
      INVARIANT_MARKER_k();  // appears at the *start of each loop body*
      ```
      - The program contains a single target property as an assertion of the form:
      ```c
      assert(<target_property>);
      ```
    - A target loop marker (e.g., "INVARIANT_MARKER_1")

    ## Task
    - Propose ONE loop invariant that is intended to hold specifically at the target loop marker.
    - The invariant should help prove the target property and be inductive if possible.

    ## Output format
    - Output MUST be a single JSON object on one line wrapped in ```json``` tags and nothing else.
    - The JSON MUST have exactly these keys:
      - "marker": MUST be exactly the target loop marker (e.g., "INVARIANT_MARKER_1")
      - "content": ONLY a valid C boolean expression for the invariant. 

    ## Output format example
    ```json
    {"marker":"<target_marker>","content":"<content>"}
    ```

User:
  ## User Input
  ### C Program
  ```c
  {program}
  ```
  ### Target Loop Marker
  {target_marker}
\end{lstlisting}
\newpage
\subsection{Prompt for Invariant Simplification}
\label{app:prompt-invariant-simplification}
We provide here the full system and user prompts used for the simplification step in the \textsc{Wonda} pipeline.
\begin{lstlisting}[
    basicstyle=\ttfamily\small,
    breaklines=true,            % WRAPS LONG LINES AUTOMATICALLY
    breakatwhitespace=true,
    columns=fullflexible,
    frame=single,
    rulecolor=\color{black!30},
    backgroundcolor=\color{gray!5},
    showstringspaces=false
]
System:
    ## Task
    Given the C program and the invariant, your task is to simplify the
    invariant to a more compact and general form.

    ## Output format
    - Output MUST be a single JSON object.
    - The JSON MUST have exactly these keys:
      - "simplified_invariant": A single compact, inductive, C boolean
        expression, nothing else.
      - "rationale": A short explanation of why you simplified the
        invariant to the given form.
    ## Output format example
    {"simplified_invariant":"<simplified_invariant>",
        "rationale":"<rationale>"}

    ## Guidelines
    - The simplified invariant should be logically weaker than (or
      equivalent to) the original, but still inductive and strong enough
      to prove the target property.
    - Prefer LINEAR arithmetic expressions (the verifier struggles with
      non-linear math like x*y)
    - Prefer mathematical relationships over case enumeration
    - Look for patterns across disjuncts (e.g., repeated structure with
      varying constants)
    - Generalize enumerated values to ranges (e.g., "i == 1 || i == 2
      || i == 3" -> "1 <= i && i <= 3")
    - Remove tautological constraints (e.g., "a == a", "n <= n",
      "0 <= 0", "a + 0 == a", "true", "1")
    - Remove constraints on constant variables (variables initialized
      but never modified in loops)
    - Replace redundant constraints with simpler equivalents (e.g.,
      "a <= b && b <= a" -> "a == b")
    - Ensure the simplified invariant is still inductive (holds before
      loop and preserved by each iteration)
    - Use the program context to understand variable semantics and
      loop structure
    - Use ONLY plain ASCII characters in your output (no Unicode symbols)

User:
    Simplify the following invariant for the given C program and marker.
    c_program:
    ```c
    {program}
    ```
    invariant:
    ```c
    {invariant}
    ```
    
    marker:
    ```c
    {marker}
```
\end{lstlisting}
\newpage


\section{Algorithms}\label{app:algorithms}
\begin{figure}[H]
\begin{minipage}[t]{0.48\textwidth}
\begin{algorithm}[H]
\caption{Candidate Invariant Grading}
\label{alg:verify}
\begin{algorithmic}[1]
\STATE {\bfseries Input:} Verification query $\langle A, P, q \rangle$, location $l$, candidate predicate $\varphi$, baseline time $t_b$
\STATE {\bfseries Output:} Quality grade $g \in \{0, 1, 2, 3\}$
\STATE
\IF{\textbf{not} $\textsc{SyntaxValid}(\varphi)$}
    \STATE {\bfseries return} $0$ \COMMENT{Invalid syntax}
\ENDIF
\STATE
\STATE $I \leftarrow (l, \varphi)$
\STATE \COMMENT{Parallel execution (cf.\ Figure~\ref{fig:vbs})}
\STATE $(\mathcal{V}_1, t_1) \leftarrow \mathcal{V}(A, P, I)$ \COMMENT{Correctness Check}
\STATE $(\mathcal{V}_2, t_2) \leftarrow \mathcal{V}(A \cup \{I\}, P, q)$ \COMMENT{Sufficiency Check}
\STATE $t_v \leftarrow \max(t_1, t_2)$ \COMMENT{Total wall-clock time}
\STATE
\IF{$\mathcal{V}_1 \ne \textsc{True}$}
    \STATE {\bfseries return} $0$ \COMMENT{Incorrect (not inductive)}
\ELSIF{$\mathcal{V}_2 \ne \textsc{True}$}
    \STATE {\bfseries return} $1$ \COMMENT{Correct but not sufficient}
\ELSIF{$t_v \ge t_b$}
    \STATE {\bfseries return} $2$ \COMMENT{Correct and sufficient, but no speedup}
\ELSE
    \STATE {\bfseries return} $3$ \COMMENT{Correct, sufficient, and provides speedup}
\ENDIF
\end{algorithmic}
\end{algorithm}
\end{minipage}
\hfill
\begin{minipage}[t]{0.48\textwidth}
\begin{algorithm}[H]
\caption{Invariant Simplification}
\label{alg:simplify}
\begin{algorithmic}[1]
\STATE {\bfseries Input:} Verification query $\langle A, P, q \rangle$, location $l$, normalized invariant $\varphi_{\text{norm}}$, baseline time $t_b$, $N$ candidates, minimum character length $\eta$
\STATE {\bfseries Output:} Set $R$ of qualifying simplified invariants with their corresponding grades.
\STATE $R \leftarrow \emptyset$
\IF{$\varphi_{\text{norm}} \in \{0, 1\}$}
    \STATE \textbf{return} $R$
\ENDIF
\IF{$|\varphi_{\text{norm}}| > \eta$} 
    \STATE $\mathbf{C} \leftarrow \textsc{LLM}(P, \varphi_{\text{norm}}, l, N)$
    \STATE $\mathbf{C} \leftarrow \textsc{Deduplicate}(\mathbf{C})$
    \FOR{\textbf{each} $\varphi \in \mathbf{C}$}
        \IF{$\varphi \in \{0, 1\}$}
            \STATE \textbf{continue} 
        \ENDIF
        \STATE $g \leftarrow \textsc{GradeCandidate}(\langle A, P, q \rangle, l, \varphi, t_b)$ \COMMENT{via Alg.~\ref{alg:verify}}
        \IF{$g \geq 2$}
            \STATE $R \leftarrow R \cup \{ (\varphi, g) \}$
        \ENDIF
    \ENDFOR
\ENDIF
\IF{$R = \emptyset$} 
    \STATE $g \leftarrow \textsc{GradeCandidate}(\langle A, P, q \rangle, l, \varphi_{\text{norm}}, t_b)$
    \IF{$g \geq 2$}
        \STATE $R \leftarrow \{ (\varphi_{\text{norm}}, g) \}$
    \ENDIF
\ENDIF
\STATE \textbf{return} $R$
\end{algorithmic}
\end{algorithm}
\end{minipage}
\end{figure}


\newpage
\section{Training Data Statistics}\label{app:training-data-statistics}
\paragraph{Pipeline Yield.}
We applied \textsc{Wonda} to 4{,}000 raw verifier-generated invariants. Of these, 3{,}932 (98.3\%) produced at least one accepted candidate ($G(\varphi) > 0$); only 68 (1.7\%) yielded an empty set. Of the 2{,}995 verbose invariants ($|\varphi_{\text{norm}}| \geq 20$), the LLM simplification stage generated 11{,}980 candidates ($N{=}4$), retaining 6{,}584 (55.0\%) with $G(\varphi) \geq 2$, along with 187 $G(\varphi){=}1$ and 39 normalized fallbacks. The remaining 1{,}005 compact invariants were verified as-is; 983 (97.8\%) yielded a $G(\varphi) \geq 2$ candidate. The pipeline produces 7{,}763 curated rows in total; restricting to $G(\varphi) \geq 2$ and dropping examples whose full prompt-completion sequence (after applying the chat template) exceeds 1{,}024 tokens gives the final \textit{V2} set of 7{,}284 samples.

\paragraph{Dataset Statistics.}
As shown in Figure~\ref{fig:train_stats}, the \textit{V2} set is partitioned 80/20 into 5{,}827 training and 1{,}457 validation samples, comprising two quality tiers: Correct and Sufficient ($G(\varphi)=2$, 4{,}516 samples) and Provides Speedup ($G(\varphi)=3$, 2{,}768 samples). The mean sequence length is 518 tokens; generated invariants are highly concise at 15.8 tokens on average. The $G(\varphi)=3$ subset achieves a mean speedup of 2.13$\times$ with peaks up to 41.39$\times$, demonstrating that \textsc{Wonda}-curated invariants offer significant computational advantages for the formal solver.

\begin{figure}[H]
    \centering
    \includegraphics[width=0.8\linewidth]{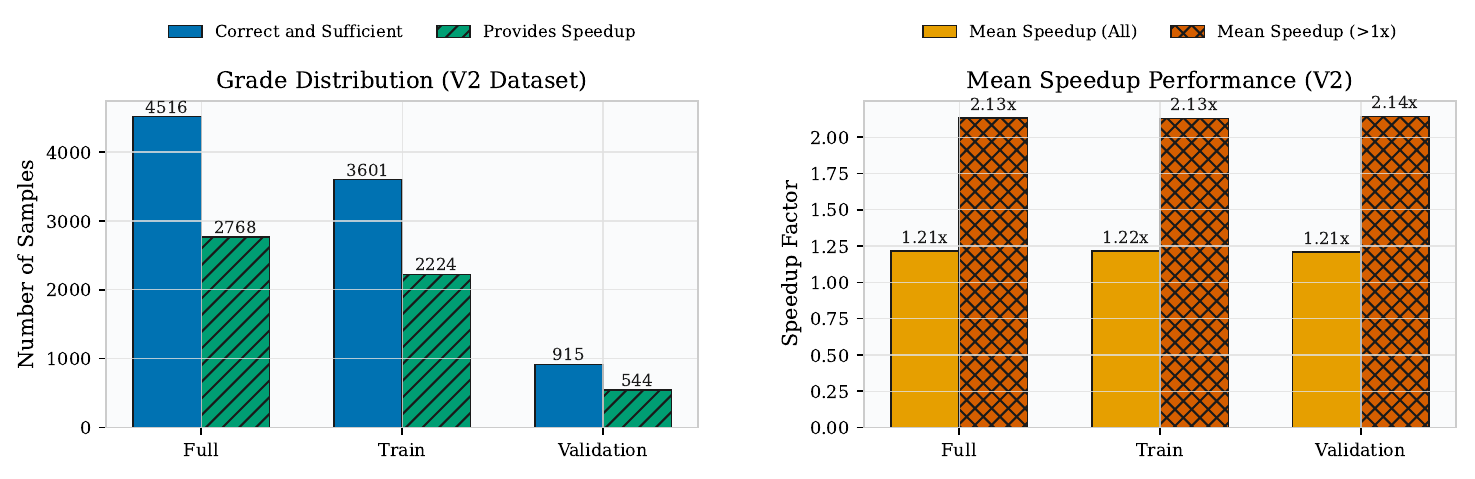}
    \caption{Statistical distribution of the \textit{V2} curated dataset across the 80/20 train-validation split.}
    \label{fig:train_stats}
\end{figure}

\section{Hard Split Evaluation}\label{app:baseline-hard-cases}
To quantify the difficulty of the hard cases, we evaluate the baseline verifier on this subset 3 times and use the median timing among them. \cref{fig:dataset-statistics-v2-hard} (Left) reports the verifier’s performance on the 123 selected hard instances. The verifier fails to solve 16.3\% of these cases (compared to 5.5\% across all instances), confirming that these represent substantially more challenging verification problems. 
\cref{fig:dataset-statistics-v2-hard} (Right) shows the runtime distribution of the baseline verifier on the 123 hard cases. Execution times range from 15.5\,s to the 600\,s timeout, with a median of 112.2\,s and a mean of 193.0\,s.

\begin{figure}[H]
  \centering
  \includegraphics[width=\textwidth]{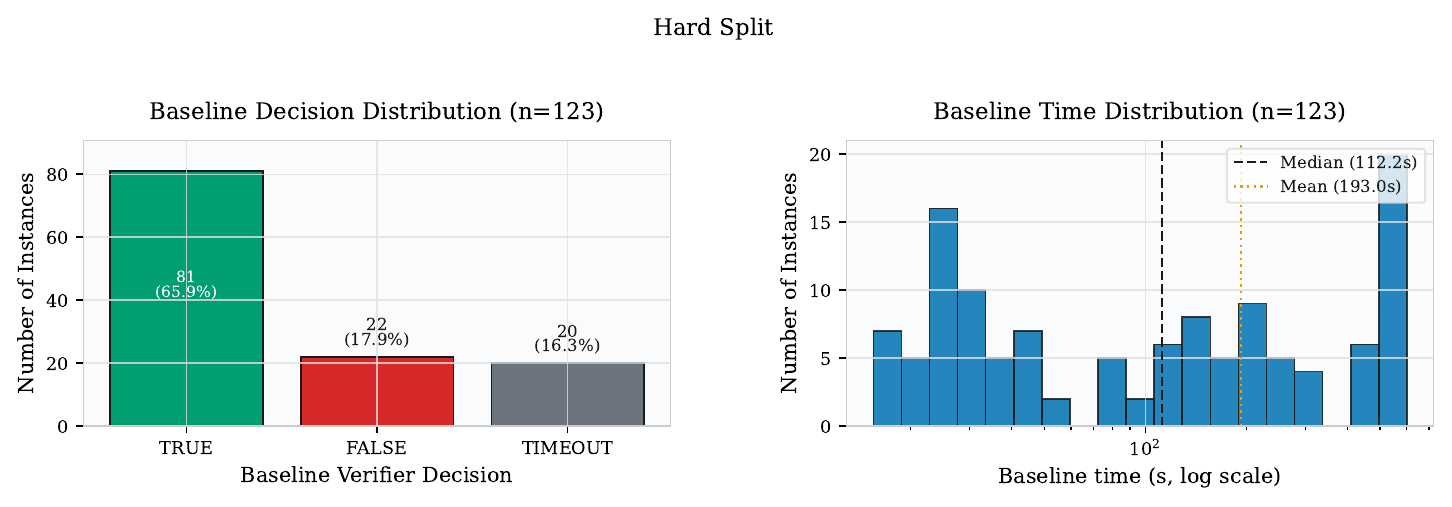}
  \caption{%
    \textbf{Hard evaluation split baseline characterization.}
    72 programs expanded to 123 per-loop instances
    (\texttt{median\_timing}${>}15$\,s).
    \textbf{(a)}~Baseline verifier decisions ($n{=}123$ instances).
    \textbf{(b)}~Baseline time distribution ($n{=}123$; instance-level baseline VBP).
  }
  \label{fig:dataset-statistics-v2-hard}
\end{figure}

\subsection{Benchmark Characterization}
\label{app:benchmark-characterization}

The hard evaluation split comprises 72 SV-COMP programs (which expands to 123 per-loop instances).
\cref{fig:benchmark-characterization} summarizes the code structure:
\begin{figure*}[ht]
  \centering
  \includegraphics[width=\textwidth]{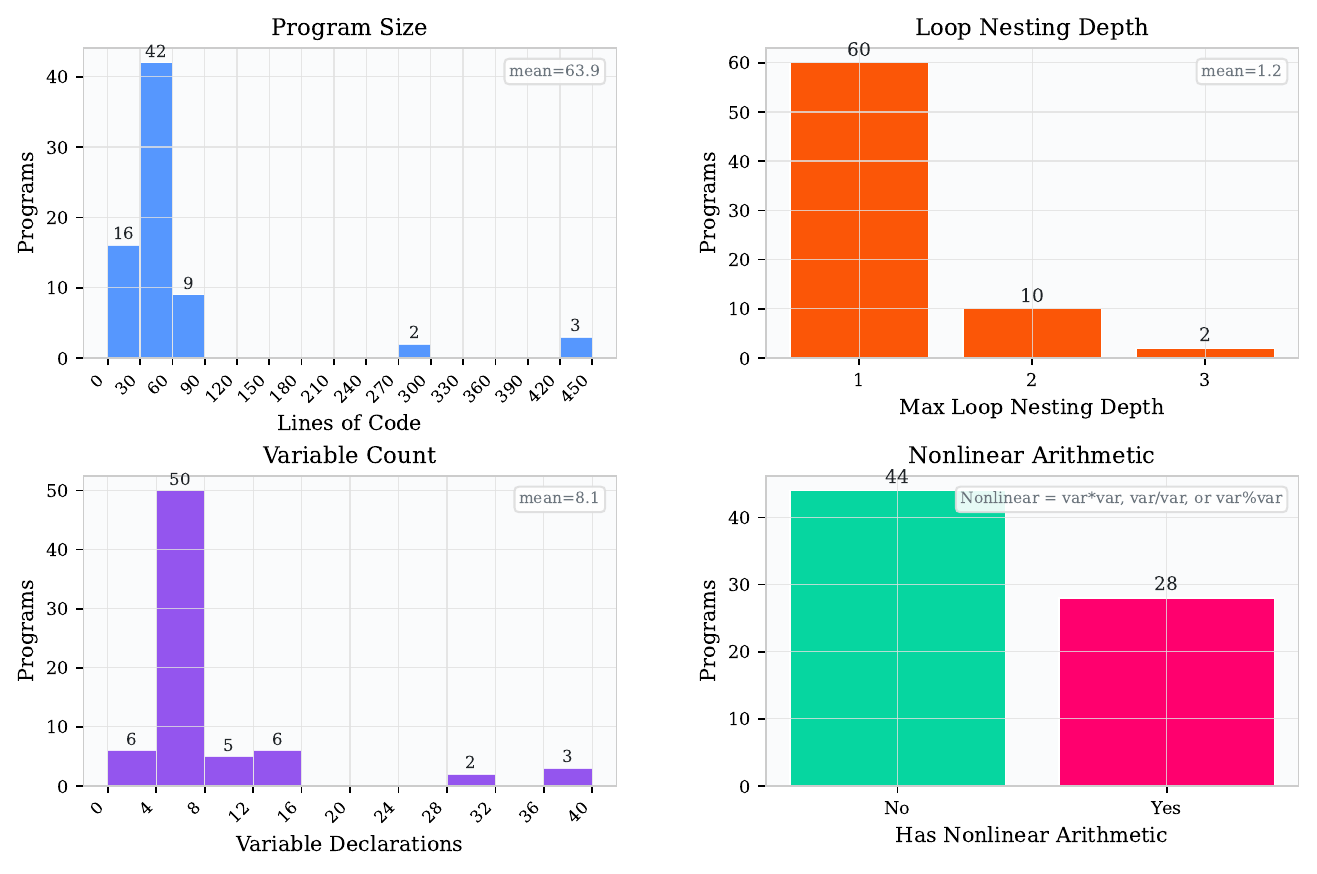}
  \caption{%
    \textbf{Benchmark characterization on hard programs} ($n{=}72$ unique
    C programs).
  }
  \label{fig:benchmark-characterization}
\end{figure*}

\subsection{Post-hoc Timeout Sweep}
\label{app:timeout-sweep}
We performed a post-hoc timeout sweep ($T \in \{15, 30, \ldots, 600\}\,\mathrm{s}$), shown in \cref{fig:timeout-sweep}.
At each~$T$, rates count only instances that finish within~$T$.

\textsc{Wonda}-V2 dominates its base counterpart at every timeout on all three panels.
$R_{\mathrm{correct}}$ and $R_{\mathrm{speedup}}$ rise with~$T$ for both, but V2 climbs faster and plateaus substantially higher; gains are already visible at short timeouts.
$\mathrm{VBP}_{\mathrm{E2E}}$ stays lower for V2 throughout, while base models remain near the ${\sim}193$\,s verifier baseline.

\begin{figure*}[ht]
  \centering
  \includegraphics[width=\textwidth]{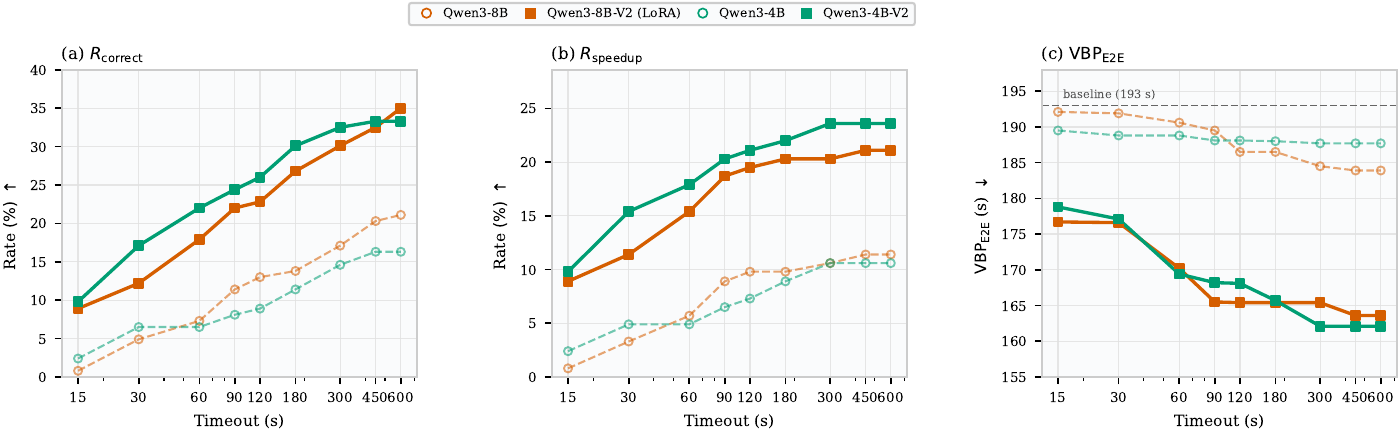}
  \caption{%
    \textbf{Post-hoc timeout sweep on hard instances} ($n{=}123$;
    $T \in \{15, 30, \ldots, 600\}\,\mathrm{s}$).
    \textbf{(a)}~$R_{\mathrm{correct}}$,
    \textbf{(b)}~$R_{\mathrm{speedup}}$, and
    \textbf{(c)}~$\mathrm{VBP}_{\mathrm{E2E}}$ vs.\ $T$
    for Qwen3-4B/8B base (circles, dashed) and \textsc{Wonda}-V2 (squares, solid;
    8B-V2 uses LoRA).
    Dashed line in~(c): solver baseline (${\approx}193$\,s).
  }
  \label{fig:timeout-sweep}
  \label{fig:R1}  
\end{figure*}

\newpage
\section{Model training details}\label{app:training-details}
\begin{table}[H]
\centering
\footnotesize
\setlength{\tabcolsep}{5pt}
\renewcommand{\arraystretch}{1.1}
\caption{Full SFT hyperparameters for all other fine-tuned models on invariant generation (no-think mode).}
\label{tab:sft-hyperparams-full}
\begin{tabular}{@{}l c c c@{}}
\toprule
\textbf{Hyperparameter} &
\shortstack[c]{\textbf{Qwen3-0.6B /}\\\textbf{4B}} &
\shortstack[c]{\textbf{Mistral-7B /}\\\textbf{Llama3.1-8B}} &
\textbf{Qwen3-14B} \\
\midrule
Learning Rate & $1 \times 10^{-4}$ & $2 \times 10^{-5}$ & $5 \times 10^{-5}$ \\
LR Scheduler &
\shortstack[c]{\texttt{cosine\_with\_min\_lr}\\(min ratio: 0.1)} &
\texttt{cosine} &
\shortstack[c]{\texttt{cosine\_with\_min\_lr}\\(min ratio: 0.1)} \\
Warmup Ratio & 0.03 & 0.05 & 0.03 \\
Epochs & 2 & 3 & 3 \\
Effective Batch Size & 32 & 32 & 32 \\
Max Seq.\ Length & 1024 & 1024 & 1024 \\
\bottomrule
\end{tabular}
\end{table}

\begin{table}[H]
\centering
\small
\caption{Supervised Fine-Tuning (SFT) Hyperparameters for Qwen3-8B-V2 on invariant generation with LoRA~\citep{Hu2022LoRALA}.}
\label{tab:sft-hyperparams-qwen8b-lora}
\begin{tabular}{@{}llp{8cm}@{}}
\toprule
\textbf{Category} & \textbf{Hyperparameter} & \textbf{Value / Setting} \\
\midrule
\textit{Optimizer \& LR} & Learning Rate & $5 \times 10^{-4}$ \\
& LR Scheduler & \texttt{cosine\_with\_min\_lr} (min ratio: 0.1) \\
& Warmup Ratio & 0.03 \\
\midrule
\textit{Batch Config} & Epochs & 2 \\
& Effective Batch Size & 32 \\
& Max Seq.\ Length & 1024 tokens \\
\midrule
\textit{LoRA (PEFT)} & Rank ($r$) & 128 \\
& Alpha ($\alpha$) & 64 \\
& Dropout & 0.05 \\
& Target Modules & All linear layers + \texttt{embed\_tokens} \\
\bottomrule
\end{tabular}
\end{table}

\begin{table}[ht]
\centering
\small
\begin{subtable}[t]{0.48\textwidth}
\centering
\caption{Qwen3 non-think models (0.6B--14B).}
\label{tab:qwen-nt-sampling}
\begin{tabular}{@{}ll@{}}
\toprule
\textbf{Hyperparameter} & \textbf{Value} \\
\midrule
Max New Tokens & 1024 \\
Temperature & 0.7 \\
Top-$p$ & 0.8 \\
Top-$k$ & 20 \\
Min-$p$ & 0 \\
Repetition Penalty & 1.1 \\
\bottomrule
\end{tabular}
\end{subtable}
\hfill
\begin{subtable}[t]{0.48\textwidth}
\centering
\caption{Mistral-7B-Instruct-v0.3 and Llama-3.1-8B-Instruct.}
\label{tab:mistral-llama-sampling}
\begin{tabular}{@{}lcc@{}}
\toprule
\textbf{Hyperparameter} &
\textbf{Mistral-7B} &
\textbf{Llama-3.1-8B} \\
& \textbf{Instruct} & \textbf{Instruct} \\
\midrule
Max New Tokens & 1024 & 1024 \\
Temperature & 0.6 & 0.6 \\
Top-$p$ & 0.9 & 0.9 \\
\bottomrule
\end{tabular}
\end{subtable}
\caption{Sampling hyperparameters used for each model.}
\label{tab:sampling-params}
\end{table}

\newpage
\section{Hardware, UAutomizer Release \& Configuration}\label{app:uautomizer_config}

We used UAutomizer release for SVCOMP-2025~\citep{heizmann2013ultimate} with the configuration shown in Table~\ref{tab:uautomizer-config}.


\begin{table}[h]
\centering
\footnotesize
\setlength{\tabcolsep}{4pt}
\renewcommand{\arraystretch}{0.95}
\caption{UAutomizer Configuration}
\label{tab:uautomizer-config}
\begin{tabular}{@{}lp{0.58\linewidth}@{}}
\toprule
\textbf{Parameter} & \textbf{Value} \\
\midrule
Version & 0.3.0-dev-d790fec\footnotemark\ (Java~21, jdk-21.0.1) \\
Property & \texttt{unreach-call.prp} \\
Property specification &
\texttt{CHECK( init(main()), LTL(G ! call(reach\_error())) )} \\
Architecture & 32-bit \\
Memory limit & 16\,GB (enforced via \texttt{runlim}\footnotemark) \\
Timeout & 600\,s per verification task \\
\bottomrule
\end{tabular}
\end{table}

\addtocounter{footnote}{-1} 
\footnotetext{\url{https://zenodo.org/records/14209043}}
\stepcounter{footnote} 
\footnotetext{\url{https://github.com/arminbiere/runlim}}

\textbf{Hardware.}
\label{sec:hardware}
Experiments ran on a Linux SLURM cluster node with 8 AMD EPYC 9354 cores, 256 GB RAM, and one NVIDIA L40S GPU. The verifier was memory-limited to 16 GB using \texttt{runlim}.

\end{document}